\documentclass[lettersize,journal]{IEEEtran}
\usepackage{amsmath,amsfonts}
\usepackage{algorithmic}
\usepackage{algorithm}
\usepackage{array}
\usepackage[caption=false,font=normalsize,labelfont=sf,textfont=sf]{subfig}
\usepackage{textcomp}
\usepackage{stfloats}
\usepackage{url}
\usepackage{verbatim}
\usepackage{graphicx}
\usepackage{cite}
\usepackage{pifont}
\usepackage[colorlinks=true, linkcolor=black, citecolor=black, urlcolor=blue]{hyperref}
\usepackage{multirow}%
\usepackage{amssymb}%
\usepackage{amsthm}%
\usepackage{mathrsfs}%
\usepackage{xcolor}%
\usepackage{manyfoot}%
\usepackage{booktabs}%
\usepackage{listings}%
\usepackage{adjustbox}
\usepackage{tikz}
\usepackage{mathabx}
\usepackage{anyfontsize}
\usetikzlibrary{spy}
\usepackage{subcaption}
\newtheorem{theo}{Theorem}[section]

\theoremstyle{definition}

\theoremstyle{remark}
\newtheorem{remark}[theo]{Remark}

\def\div{\text{div}}

\newcommand{\cmark}{\ding{51}} 
\newcommand{\xmark}{\ding{55}} 

\newcommand{\Rb}{\ensuremath{\mathbb{R}}}
\def\prox{\text{prox}}
\def\div{\text{div}}

\newlength{\spyimagewidth}
\newlength{\spyimageheight}

\newcommand{\spyimageStagestest}[2][0.16\linewidth]{%
  \setlength{\spyimagewidth}{#1}
  \setlength{\spyimageheight}{\spyimagewidth}
  \begin{tikzpicture}
    \node[inner sep=0pt, anchor=south west] (image) at (0,0) {%
      \adjustbox{trim={.33\width} {0.23\height} {0.1\width} {0.2\height},clip,width=\spyimagewidth}{%
        \includegraphics{#2}%
      }%
    };
  \end{tikzpicture}%
}

\newcommand{\spyimageto}[2][0.43\linewidth]{%
  \setlength{\spyimagewidth}{#1}
  \setlength{\spyimageheight}{\spyimagewidth}
  \begin{tikzpicture}[spy using outlines={rectangle, magnification=1.8, size=0.4\spyimagewidth, connect spies,color=white}]
    \node[inner sep=0pt, anchor=south west] (image) at (0,0) {\adjustbox{trim={.1\width} {0.1\height} {0.2\width} {0.2\height},clip,width=\spyimagewidth}{\includegraphics{#2}}};
   
  \end{tikzpicture}%
}

\newcommand{\spyimagegradfidel}[2][0.43\linewidth]{%
  \setlength{\spyimagewidth}{#1}
  \setlength{\spyimageheight}{\spyimagewidth}
  \begin{tikzpicture}[spy using outlines={rectangle, magnification=2.7, size=0.4\spyimagewidth, connect spies,color=white}]
    \node[inner sep=0pt, anchor=south west] (image) at (0,0) {\adjustbox{trim={.1\width} {0.1\height} {0.1\width} {0.1\height},clip,width=\spyimagewidth}{\includegraphics{#2}}};
    \coordinate (spy point) at (0.42\spyimagewidth , 0.15\spyimagewidth );
       \coordinate (spy node) at (0.2\spyimagewidth, 0.5\spyimageheight);
    \spy on (spy point) in node at (spy node);
  \end{tikzpicture}%
}

\newcommand{\spyimageRetinex}[2][0.43\linewidth]{%
  \setlength{\spyimagewidth}{#1}
  \setlength{\spyimageheight}{\spyimagewidth}
  \begin{tikzpicture}[spy using outlines={rectangle, magnification=2.7, size=0.4\spyimagewidth, connect spies,color=white}]
    \node[inner sep=0pt, anchor=south west] (image) at (0,0) {\adjustbox{trim={.1\width} {0.1\height} {0.1\width} {0.1\height},clip,width=\spyimagewidth}{\includegraphics{#2}}};
  \end{tikzpicture}%
}

\newcommand{\spyimageMHA}[2][0.414\linewidth]{%
  \setlength{\spyimagewidth}{#1}
  \setlength{\spyimageheight}{\spyimagewidth}
  \begin{tikzpicture}[spy using outlines={rectangle, magnification=2.7, size=0.4\spyimagewidth, connect spies,color=white}]
    \node[inner sep=0pt, anchor=south west] (image) at (0,0) {\adjustbox{trim={.\width} {0.\height} {0.\width} {0.\height},clip,width=\spyimagewidth}{\includegraphics{#2}}};
  \end{tikzpicture}%
}

\newcommand{\spyimageUIEB}[2][0.2165\linewidth]{%
  \setlength{\spyimagewidth}{#1}
  \setlength{\spyimageheight}{\spyimagewidth}
  \begin{tikzpicture}[spy using outlines={rectangle, magnification=1.9, size=0.5\spyimagewidth, connect spies,color=white}]
    \node[inner sep=0pt, anchor=south west] (image) at (0,0) {\adjustbox{trim={0.3\width} {0.1\height} {0.\width} {0.2\height},clip,width=\spyimagewidth}{\includegraphics{Results/UIEB/3/#2}}};
    \coordinate (spy point) at (0.85\spyimagewidth , 0.42\spyimagewidth );
       \coordinate (spy node) at (0.3\spyimagewidth, 0.75\spyimageheight);
    \spy on (spy point) in node at (spy node);
  \end{tikzpicture}%
}

\newcommand{\spyimageSUIME}[2][0.2165\linewidth]{%
  \setlength{\spyimagewidth}{#1}
  \setlength{\spyimageheight}{\spyimagewidth}
  \begin{tikzpicture}[spy using outlines={rectangle, magnification=4, size=0.4\spyimagewidth, connect spies,color=white}]
    \node[inner sep=0pt, anchor=south west] (image) at (0,0) {\adjustbox{trim={.2\width} {0.\height} {0.\width} {0.2\height},clip,width=\spyimagewidth}{\includegraphics{Results/SUIME/10/#2}}};
    \coordinate (spy point) at (0.9\spyimagewidth , 0.06\spyimagewidth );
       \coordinate (spy node) at (0.85\spyimagewidth, 0.7\spyimageheight);
    \spy on (spy point) in node at (spy node);
  \end{tikzpicture}%
}

\begin{document}

\title{Variational Deep Unfolding with Mamba-Based Nonlocal Modeling for Underwater Image Enhancement}

\author{Daniel Torres, Julia Navarro, Catalina Sbert, and Joan Duran
        \thanks{The authors are with Institute of Applied Computing and Community Code (IAC3) and with the Dept.~of Mathematics and Computer Science, Universitat de les Illes Balears, Cra.~de Valldemossa km.~7.5, E-07122 Palma, Spain.}
\thanks{This work is part of the MoMaLIP project PID2021-125711OB-I00 funded by MICIU/AEI/10.13039/501100011033 and the European Union NextGeneration EU/PRTR. Daniel Torres is also supported by the Conselleria d’Educació i Universitats del GOIB through grant FPU2024-002-C. The authors gratefully acknowledge the computer resources at Artemisa, funded by the EU ERDF and Comunitat Valenciana and the technical help from IFIC (CSIC-UV).}
}

\markboth{}%
{Daniel \MakeLowercase{\textit{et al.}}: Variational Deep Unfolding with Mamba-Based Nonlocal Modeling for Underwater Image Enhancement}


\maketitle

\begin{abstract}
Underwater imaging plays a crucial role in ocean engineering, although captured data often suffer from poor visibility and color distortion. To address these challenges, we propose a model-based deep unfolding network for underwater image enhancement that integrates variational modeling into a learnable architecture. The framework is guided by a variational formulation based on a dehazing decomposition, incorporating a multiplicative residual component to absorb remaining artifacts and a nonlocal gradient-type constraint to preserve structural details and enhance edge sharpness. We provide a theoretical analysis establishing the existence of solution for the associated minimization problem. The proposed unfolding method incorporates Mamba layers to efficiently capture self-similarities in the scene. In addition, we introduce a proximal trajectory loss that enforces consistency between the unfolding stages and the iterations of an ideal restoration regularizer. Experimental results demonstrate that the proposed unfolding approach achieves improved visual quality and competitive quantitative performance compared with recent state-of-the-art methods. The source
code will be available at \href{https://github.com/MIA-UIB/Variational-Unfolding-Mamba-Underwater-Enhancement}
     {this repository}.
\end{abstract}


\section{Introduction}\label{sec1}

\IEEEPARstart{T}{he} acquisition of clear underwater images is of major importance in ocean engineering and marine research, where different devices, such as fixed cameras or underwater vehicles, are widely used to explore marine environments. However, raw underwater images usually suffer from limited visibility, non-uniform illumination, color attenuation, and noise, which significantly degrade visual quality and affect subsequent vision tasks \cite{raveendran2021underwater}.

Classical underwater image enhancement approaches can be broadly divided into model-free and physical model-based methods. Model-free techniques \cite{iqbal2010enhancing, ancuti2012enhancing, ancuti2017} directly manipulate pixel intensities to improve contrast, brightness, or color balance without relying on physical assumptions. Model-based methods formulate the restoration process as an inverse problem governed by the underwater image formation model. Although several authors have proposed approaches relying on the Retinex theory \cite{Retinex}, the most typical assumption is to model underwater images as hazy observations. According to the Jaffe-McGlamery model \cite{McGlamery80, Jaffe90}, a hazy image consists of a linear superposition of the direct, backscattering, and forward scattering components. Within the variational framework, the solution of the inverse problem is obtained by minimizing an energy functional that comprises data-fidelity and regularization terms \cite{hou2019efficient, hou2020novel, xie2022variational, li2025dual}. However, the effectiveness of these methods strongly depends on the choice of suitable priors.

In contrast, deep learning techniques \cite{li2017watergan, li2019underwater, li2021underwater} overcome this limitation by learning priors directly from data, although they often rely on highly complex architectures and large training datasets. To bridge the gap between model-based optimization and deep learning, unfolding networks \cite{pham2023deep, xu2025degradation} map the iterations of classical optimization algorithms into neural architectures by replacing specific operators, such as proximal mappings in first-order minimization schemes \cite{chambolle2016introduction}, with learnable modules. This results in more interpretable and computationally efficient models.

In this work, we propose a model-based deep unfolding framework for underwater image enhancement. The method is guided by a variational formulation based on a dehazing decomposition, incorporating multiplicative residual noise and a nonlocal gradient-type constraint. The proximal operators are replaced by learnable neural networks, while Mamba mechanisms \cite{gu2024mamba} are integrated to efficiently capture self-similarities. In addition, we introduce a proximal trajectory loss that enforces consistency between the unfolding stages and the iterations of an ideal restoration regularizer. Specifically, our main contributions are summarized below:
\begin{itemize}
\item A variational formulation for underwater image enhancement based on a dehazing decomposition, integrating a multiplicative residual component  to absorb remaining artifacts and a nonlocal gradient-type constraint to preserve structural details and enhance edge sharpness.
\item A theoretical analysis proving the existence of minimizers for the proposed variational problem.
\item A model-based deep unfolding network integrating Mamba layers to efficiently capture long-range dependencies in underwater scenes.
\item A proximal trajectory loss specifically designed for the underwater imaging setting to enforce consistency between the unfolding stages and the iterations of the corresponding optimization algorithm.
\end{itemize}

The rest of the paper is organized as follows. Section 2 reviews the state of the art in underwater image enhancement. Section 3 presents the variational baseline of our method, while Section 4 describes the associated unfolding network. Experimental comparisons on the UIEB \cite{li2019underwater} and SUIM-E \cite{qi2022sguie} datasets are reported in Section 5, followed by an ablation study in Section 6. Conclusions are drawn in Section 7.

\section{State of the Art}

\subsection{Classical methods}

Model-free techniques aim to improve underwater image quality by directly adjusting pixel values without relying on a physical image formation model.
One of the earliest approaches to address color distortion was proposed by Iqbal et al.~\cite{iqbal2010enhancing}. Their method compensates for dominant color casts by equalizing individual channels, followed by histogram stretching to enhance contrast in both RGB and HSI color spaces.
Hitam et al.~\cite{hitam2013mixture} enhance images through Contrast Limited Adaptive Histogram Equalization (CLAHE) in both RGB and HSV spaces. Then, the resulting images are combined to improve contrast while mitigating noise amplification.

Ancuti et al.~\cite{ancuti2012enhancing} proposed a fusion-based framework in which several enhanced versions of the input image are generated and combined through multi-scale fusion weights.
The authors later extended this approach in \cite{ancuti2017}, where color balance and multi-scale fusion are more systematically combined to improve robustness across diverse underwater conditions.

Fusion-based approaches have also continued to evolve during recent years. An et al.~\cite{an2024hfm} proposed fusing underwater images after visibility recovery and contrast enhancement, considering both pixel intensity and global gradient changes. Zhang et al.~\cite{zhang2022underwater} proposed a method based on a minimal color loss principle and locally adaptive contrast enhancement. The same authors also generate global and local  contrast-enhanced images through separate strategies and fuse their multi-scale frequency components using a wavelet-based method \cite{zhang2024underwater}. 

Classical model-based methods rely on explicit physical descriptions of the degradation process. A common assumption interprets underwater images as hazy observations following the Jaffe-McGlamery model \cite{McGlamery80, Jaffe90}, where the observed image is expressed as a linear superposition of direct transmission, background scattering, and forward scattering components:
\begin{equation*}E_T= E_d + E_b + E_f,\end{equation*}
where $E_T$ is the total irradiance that reaches the camera. The direct transmission component describes the attenuated scene radiance and is commonly written as $E_d=Jt$, where $J$ denotes the restored image and $t$ is the transmission map. The background scattering term is typically represented as $E_b=A(1-t)$, where $A$ corresponds to the global backscattered light. The forward scattering component $E_f$ models the spatial blur caused by light diffusion in the water medium and is usually represented through a convolution with a distance-dependent point spread function. However, it is often ignored \cite{ancuti2017} and the image observation model can be simplified as
\begin{equation}\label{model_haze}
I= Jt+A(1-t),
\end{equation}
where $I$ denotes the observed degraded image.





Different algorithms based on \eqref{model_haze} have been proposed \cite{chiang2011underwater, serikawa2014underwater, li2016underwater}, where classical enhancement operations such as filtering, histogram manipulation, or contrast correction are not directly applied to the observed image, but instead to previously estimated components such as $t$ or $A$. Nevertheless, the most common strategy in this context consists of formulating the dehazing problem within a variational framework.

Hou et al.~\cite{hou2019efficient} propose a nonlocal prior to estimate $J$, but they omit the transmission map in the variational formulation. They extend this approach by incorporating total variation priors for both $J$ and $t$ \cite{hou2020novel}. Xie et al.~\cite{xie2022variational} consider the complete underwater image formation model to jointly remove haze and blur using a forward scattering term and a normalized total variation prior. More recently, higher-order regularization strategies have been explored, combining first- and second-order differential operators \cite{li2025dual}.

Another relevant vision model is Retinex \cite{Retinex}, which describes the ability of the human visual system to discount illumination variations when observing a scene under different lighting conditions. Galdran et al.~\cite{dehaze-retinex} theoretically showed that applying Retinex to inverted intensities provides a solution to the dehazing problem. This relation can be expressed as
\begin{equation}
    \label{retinex_dehazing}
\text{dehaze} (I)=1-\text{Retinex}(1-I).
\end{equation}
Inspired by this connection, some authors propose enhancement methods based on the Retinex theory, either through variational formulations \cite{fu2014retinex,  zhuang2021bayesian, zhuang2022underwater} or through multi-scale Retinex techniques \cite{zhang2017underwater, zhou2022multi}. 

\subsection{Deep learning methods}

The success of deep learning techniques has significantly influenced underwater image enhancement, leading to a wide variety of data-driven approaches that depart from handcrafted priors. Early methods addressed the scarcity of paired data using adversarial learning \cite{li2017watergan, fabbri2018enhancing}.
With the release of paired datasets, supervised CNN-based methods became dominant. In particular, Li et al.~\cite{li2019underwater} introduced the widely used UIEB dataset and proposed a CNN architecture based on the adaptive fusion of multiple pre-enhanced inputs.

Several supervised networks followed, aiming to improve color fidelity and contrast while maintaining relatively simple structures. UIEC$^2$-Net \cite{wang2021uiec} introduces a dual color-space strategy that jointly exploits RGB and HSV representations to better handle color and luminance distortions.  RAUNENet \cite{peng2023raune} proposes a residual architecture containing an attention module that combines channel and spatial attention. Shallow-UWNet \cite{naik2021shallow} and LiteEnhanceNet~\cite{zhang2024liteenhancenet}  focus on computational efficiency by proposing lightweight CNN architectures.

Beyond purely data-driven CNNs, several works attempt to incorporate elements of underwater image formation models into deep architectures. UColor~\cite{li2021underwater}  learns representative features from a multi-color space using channel attention and introduces a medium transmission as a spatial attention map.  P2CNet~\cite{rao2023deep} departs from direct enhancement by modeling color compensation through a probabilistic framework.

Transformers have shown strong performance in computer vision due to their ability to capture long-range dependencies and global context. For underwater image enhancement, methods such as AutoEnhancer \cite{tang2022autoenhancer} and Phaseformer \cite{khan2025phaseformer} adopt a U-Net–Transformer hybrid design and frequency-domain attention, respectively. However, self-attention incurs high computational and memory costs. Alternative approaches, like UIR-PolyKernel \cite{guo2025underwater}, employ large-kernel CNNs to efficiently capture long-range dependencies without explicit attention.

State Space Models (SSMs) have emerged as an effective alternative for modeling global context with linear complexity. The Mamba architecture \cite{gu2024mamba} builds on selective state space mechanisms, enabling efficient long-range modeling. Variants applied to underwater enhancement include UW-Mamba \cite{an2024uwmamba}, which combines visual state space blocks with convolutions for local perception, and Mamba-UIE \cite{zhang2024mamba}, which integrates the underwater formation model as a physical loss function.

Unfolding techniques provide a powerful framework to incorporate model information into deep networks by mapping the iterations of energy-based optimization algorithms into neural architectures. In underwater enhancement, several works unfold general degradation models of the form \begin{equation*}\frac{1}{2}\|I-DJ\|_2^2+\lambda \Phi(J),\end{equation*}
 where $\Phi(J)$
 is a learnable regularizer. For instance, \cite{xu2025degradation} assumes  $D$
as the identity, while Lei et al.~\cite{lei2024uie} learn the operator during training. In contrast, Pham et al.~\cite{pham2023deep} unfold the degradation model \eqref{model_haze}, replacing the proximal operators associated with $J$, $t$ and $A$ with dense residual networks.



\section{Variational Baseline}

We develop a variational formulation for underwater image enhancement based on the decomposition:
 \begin{equation}\label{model}
I= (J+N)t+A(1-t),
\end{equation}
where we propose to incorporate a residual component $N$ into the classical dehazing model \eqref{model_haze}. In this section, we consider explicit priors for the restored image, the transmission map, and the residual term, leading to a fully variational formulation. This serves as a baseline for evaluating the proposed unfolding approach, in which the priors are replaced by learnable modules. Furthermore, we establish the existence of minimizers for the associated optimization problem.

For further details on the functional analysis concepts omitted in this section, we refer the reader to \cite{ambrosio2000functions, adams2003sobolev}.

\subsection{Notations and definitions}
Let $\Omega\subset\mathbb{R}^d$, $d\geq 2$, be an open bounded domain with Lipschitz boundary and let $I: \Omega\to \Rb^C$ be the degraded image, where $C$ is the number of color channels. We denote the restored, residual, global light, and transmission components as $J, N, A:\Omega\to\Rb^C$ and $ t:\Omega\to [0,1]$, respectively. 

Let $\omega:\Omega\times\Omega\to \Rb $ be a non-negative weight function, and define the nonlocal gradient $\nabla_{\omega} J:\Omega\times\Omega\to \Rb^C$ for each channel $c\in \{1,2,\dots, C\}$ as
\begin{equation*}
\nabla_{\omega} J_c(x,y)=
\sqrt{\omega(x,y)}\left(J_c(y)-J_c(x)\right).
\end{equation*}
Given a non-negative vector weight function $\widehat{\omega}:\Omega\times\Omega\to \Rb^d$ and a vector field $V:\Omega\to \Rb^{d\times C} $, we define the nonlocal vector $(\nabla J-V)_{\widehat{\omega}}: \Omega\times\Omega\to \Rb^{d\times C}$ as 
\begin{equation*} (\nabla J_c-V_c)_{\widehat{\omega}}(x, y)=\sqrt{\widehat{\omega}(x,y)}\circ\left(\nabla J_c(x)-V_c(y)\right),\end{equation*}
where $\circ$ denotes the componentwise product.

Recall the $L^2$ norm for the different functions that appear in the proposed model:
\begin{equation*}
\begin{aligned}
&\|N\|_2^2=\sum_{c=1}^C \int_\Omega (N_c(x))^2 dx, \quad \|t\|_2^2=\int_\Omega (t(x))^2 dx,\\[1ex]
&\|\nabla_\omega J\|_2^2=\sum_{c=1}^C \int_{\Omega\times\Omega} \hspace{-0.45cm}\omega(x,y)(J_c(y)\hspace{-0.06cm}-\hspace{-0.06cm} J_c(x))^2 dx dy,\\[1ex]
&\|\hspace{-0.06cm} \left(\nabla J \hspace{-0.06cm}- \hspace{-0.06cm}V\right)_{\widehat{\omega}} \hspace{-0.06cm}\|_2^2=\hspace{-0.02cm}\sum_{c=1}^C\hspace{-0.048cm}\int_{\Omega\times\Omega} \hspace{-0.52cm}|\sqrt{\widehat{\omega}(x,y)}\hspace{-0.04cm}\circ\hspace{-0.04cm}\left(\nabla J_c(x)\hspace{-0.06cm}-\hspace{-0.06cm}V_c(y)\right)\hspace{-0.06cm}|^2 dx dy,
\end{aligned}
\end{equation*}
where $|\cdot|$ denotes the euclidean norm in $\Rb^d$. Moreover, the total variation of $t$ \cite{ambrosio2000functions} is denoted by $\text{TV}(t)$ .

\subsection{Proposed energy}
We propose to minimize the following energy functional for simultaneously recovering the haze- and noise-free image, the noise component and the transmission map: 
\begin{equation}
\label{energia}
\begin{aligned}
E(J,t,N)&=\frac{1}{2}\|(J+N)t+A(1-t)-I\|_2^2 \\
&+\frac{\alpha}{2} \|\nabla_\omega J\|_2^2  +\beta \text{TV}(t)+\frac{\lambda}{2} \|N\|_2^2\\
&+\frac{\mu}{2} \|(\nabla J-V)_{\widehat{\omega}}\|_2^2 +\frac{\rho}{2}\|t-t_0\|_2^2,
\end{aligned}
\end{equation}
over the admissible space $\Lambda$, defined for some $0<a<1$ as
\begin{equation*}\Lambda\hspace{-0.1cm}=\hspace{-0.1cm}\left\{(J,t,N)\in H^1(\Omega)\times BV(\Omega)\times L^2(\Omega) : t \in [a,1], a.e.\right\}.\end{equation*}

\begin{remark}
    In real captured images, a completely opaque atmosphere corresponding to $t\to 0$ is rare, since camera sensors still receive part of the direct scene radiance. Therefore, forcing  $t(x)$ to stay above a small threshold is consistent with practical physical limitations.
\end{remark}
The proposed energy relies on the following assumptions.  The first term corresponds to the decomposition model \eqref{model}. The $\alpha$-term promotes nonlocal regularity, which serves as a useful prior for preserving fine details and texture. For the weight $\omega$, we consider both the spatial closeness between points and the patch similarity in the degraded image $I$ \cite{Nlmeans}. The $\beta$-term enforces total variation on the transmission map, promoting spatial smoothness while preserving depth discontinuities. The $\rho$-term enforces fidelity to the initial transmission estimate, preventing excessive deviation from the initialization, whereas the $\lambda$-term constrains the absorption of residual components. 

The $\mu$-term penalizes the nonlocal distance between the gradient of the restored image and a given vector field $V$. For this vector field, we consider 
an amplified version of the gradient of the degraded input as in \cite{structure-Retinex}: 
\begin{equation}\label{eq-v} V=\left(1+\lambda_ge^{\frac{-|\nabla I|}{\sigma_g}}\right)\nabla I,\end{equation} where $\lambda_{g}>0$ controls the amplification strength and $\sigma_g>0$ determines its sensitivity to gradient magnitude. Since the objective of the nonlocal gradient-type fidelity term is to strengthen the structural information hidden in the degraded image, we compute the weight similarities $\widehat{\omega}$ using patch distances within the amplified image gradient.
This type of nonlocal gradient constraint has been shown to be particularly effective in addressing poor visibility and low-contrast degradation in low-light imaging \cite{ijcv_dani}. Given the strong similarities between low-light and underwater imaging in terms of contrast attenuation and visibility loss, we argue that this term is well suited for underwater image enhancement.

\subsection{Existence of minimizer}
The goal of this section is to prove the existence of a solution of the minimization problem
\begin{equation}\label{P}
\min_{(J,t,N)\in\Lambda} E(J,t,N).
\end{equation}
For this purpose , we assume that there exist $c_1, c_2>0$ such that $c_1\le \omega(x,y)\le c_2 $ a.e.~in $\Omega\times\Omega$ and $K_1, K_2>0$ such that $K_1\le \widehat{\omega}_k(x,y)\le K_2$ a.e.~in $\Omega\times\Omega$, $k=1,2,\dots,d$. 



 

\begin{theo}
Let $I\in L^{2}(\Omega; \mathbb{R}^C) $, $V\in L^\infty\big(\Omega;\Rb^{d\times C}\big)$, $t_0\in L^2(\Omega)$, and $A\in L^2(\Omega; \mathbb{R}^C).$ The problem (\ref{P})  admits at least one minimizer in $\Lambda$, for some $0<a<1$.
\end{theo}

\begin{proof}
Notice that $E(J,t,N)\geq 0$. Moreover $(0,a,0)\in \Lambda$, where $a$ means the constant function equal $a$ a.e.~in $\Omega$, and 
\begin{equation*}
\begin{aligned}E(0,a,0)=\frac{1}{2}\|(1-a) A-I\|_2^2+\frac{\mu}{2}\|(V)_{\widehat{\omega}}\|_2^2+\frac{\rho}{2}\|a-t_0\|_2^2.\end{aligned}
\end{equation*}
Then, $E(0,a,0)<+\infty$ which implies that $E$ is proper and bounded below. Therefore, the problem (\ref{P}) is well defined and there exists $m=\inf_{(J,t,N)\in\Lambda} E(J, t, N)<\infty.$ Let $(J_n, t_n, N_n)\in\Lambda$ be a minimizing sequence. Then, there exists a constant $M$ (which many change from line to line) such that
\begin{equation*}E(J_n, t_n, N_n)\leq M.\end{equation*}
Using $t_0\in L^2(\Omega)$ and $\|t-t_0\|_2^2\leq M$ we deduce that $(t_n)$ is uniformly bounded in $L^2(\Omega)$.  Because  $\Omega$ is bounded,  Hölder inequality implies  $(t_n)$ is also bounded in $L^1(\Omega)$. Together with $\text{TV}(t) \leq M$, this shows that $(t_n)$ is uniformly bounded in $BV(\Omega)$. Thus, there exist a subsequence (also denoted by $(t_n)$) and  $\hat{t}\in BV(\Omega)$ such that
\begin{equation}\label{convt}t_n\underset{L^1}{\to} \hat{t}, \qquad t_n\underset{L^2}{\rightharpoonup} \hat{t}.\end{equation}
Since, up to a subsequence, strong convergence in $L^1$ implies pointwise convergence a.e., we conclude that $\hat{t}\in[a,1]$ a.e.

Since $\|N_n\|_2^2\leq M$, we have that $(N_n)$ is uniformly bounded in $L^2(\Omega)$. Thus, there exist a subsequence (also denoted by $(N_n)$) and $\hat{N}\in L^2(\Omega)$ such that
\begin{equation}\label{convN}N_n \underset{L^2}{\rightharpoonup} \hat{N}.\end{equation}

Let us now prove that $(J_n)$ is uniformly bounded in $H^1(\Omega)$. 
First, we show that it is bounded in $L^2(\Omega)$. Writing
\begin{equation*}
J_n=\frac{1}{t_n}\Big((J_n+N_n)t_n+(1-t_n)A-I-N_nt_n-A(1-t_n)+I\Big),
\end{equation*}
and using the triangle inequality, together with 
$\|(J_n+N_n)t_n+(1-t_n)A-I\|_2^2\le M$, 
$t_n\le 1$ a.e.~in $\Omega$, $\|1/t_n\|_\infty\le 1/a$, and the fact that 
$A,I,N_n\in L^2(\Omega)$, we conclude that $(J_n)$ is uniformly bounded in $L^2(\Omega)$.

Next, we prove that $(\nabla J_n)$ is uniformly bounded in $L^2(\Omega)$. Since $\|(\nabla J_n-V)_{\widehat{\omega}}\|_2^2\leq M$, by Fubini's theorem, for almost every  $y\in\Omega$ and   every color channel $c\in\{1,2,\dots, C\}$,
\begin{equation*} \int_\Omega |\sqrt{\widehat{\omega}(x,y)}\circ\left(\nabla J_{n,c}(x)- V_c(y)\right)|^2 dx\leq M.\end{equation*}
Using the lower bound of the weight function, we get $\int_\Omega |\nabla J_{n,c}(x)- V_c(y)|^2 dx\le M$. 
Then, \begin{equation*}
    \begin{aligned}
\int_\Omega |\nabla J_{n,c}(x)|^2 dx &\leq 2 \int_\Omega |\nabla J_{n,c}(x)-V_c(y)|^2 dx\\& + 2\int_\Omega |V_c(y)|^2 dx\leq M, \end{aligned}\end{equation*}
and we can deduce that $\nabla J_n$ is uniformly bounded in $L^2(\Omega)$.
We conclude $J_n$ is uniformly bounded in $H^1(\Omega).$ Then, there exists a subsequence of $(J_{n})$ (also denoted by $(J_n)$) and  $\hat{J}\in H^1(\Omega)$ such that $J_n\rightharpoonup \hat{J}$ in $H^1(\Omega)$.
Using Rellich-Kondrachov compactness Theorem in Sobolev spaces \cite{adams2003sobolev}, up to a subsequence, we can ensure that 
\begin{equation}\label{convJ}J_{n}\underset{L^2}{\to} \hat{J}, \qquad \nabla J_{n}\underset{L^2}{\rightharpoonup} \nabla \hat{J}. \end{equation}

Since $t_n, \hat{t}\in [a, 1]$ a.e.~in $\Omega$, then $\|t_n\|_\infty\leq 1$ and $\|\hat{t}\|_\infty\leq 1$. This implies that $J_n t_n\in L^2(\Omega)$ and $\hat{J}\hat{t}\in L^2(\Omega)$. 
Now, we prove that the strong convergence of $(J_n)$ to $\hat{J}$  in $L^2(\Omega)$ (\ref{convJ}) and the strong convergence of $(t_n)$ to $\hat{t}$  in $L^1(\Omega)$ (\ref{convt}),  jointly with the fact that $t_n, \hat{t}\in L^\infty(\Omega)$, imply  that $J_n t_n\rightharpoonup \hat{J}\hat{t}$ in $L^2(\Omega)$. 
By density of $C_c^\infty(\Omega)$ in $L^2(\Omega)$ it suffices to prove $\big|\int_\Omega (J_nt_n-\hat{J}\hat{t})\varphi \, dx\big|\to 0$ for all $\varphi \in C_c^\infty(\Omega)$. Notice that \begin{equation*}\left|\int_\Omega (J_nt_n-\hat{J}\hat{t})\varphi\right|\leq\left|\int_\Omega(J_n-\hat{J})t_n\varphi\right|+\left|\int_\Omega \hat{J}(t_n-\hat{t})\varphi\right|.\end{equation*}
Since 
\begin{equation*}\begin{aligned}
\left|\int_\Omega \hat{J}(t_n-\hat{t} )\varphi\right|&\le \|\varphi\|_\infty\|t_n-\hat{t}\|_2\|\hat{J}\|_2\\&\le\|t_n-\hat{t}\|_1^{\frac{1}{2}}\|t_n-\hat{t}\|_\infty^{\frac{1}{2}}\|\varphi\|_\infty\|\hat{J}\|_2\to 0\end{aligned}\end{equation*}
and
\begin{equation*}\left|\int_\Omega (J_n-\hat{J}) t_n\varphi\right|\leq \|t_n\|_2 \|J_n-\hat{J}\|_2 \|\varphi\|_\infty\to 0,\end{equation*}
then we deduce that\begin{equation}\label{convJt}
J_n t_n  \rightharpoonup \hat{J} \hat{t}\quad \text{in } L^2(\Omega).
\end{equation}

Now, we prove that 
\begin{equation}\label{convNt}
N_n t_n  \rightharpoonup \hat{N} \hat{t}\quad \text{in } L^2(\Omega).
\end{equation}
Since $( N_n t_n)$ is uniformly bounded in $L^2(\Omega)$, it suffices to prove $\big|\int_\Omega (N_nt_n-\hat{N}\hat{t})\varphi\big| \to 0$ for all $\varphi \in C_c^\infty(\Omega)$. Notice that
\begin{equation*}\left|\int_\Omega (N_nt_n-\hat{N}\hat{t})\varphi\right|\leq \left| \int_\Omega N_n(t_n-\hat{t}) \varphi \right|+\left|\int_\Omega \hat{t}(N_n-\hat{N})\varphi\right|.\end{equation*}
Using the strong convegence of $(t_n)$ in $L^1(\Omega)$ (\ref{convt}), we have
\begin{equation*}
\begin{aligned}
\left| \int_\Omega N_n(t_n-\hat{t}) \varphi \right|
&\le
\|N_n\|_2 \|t_n-\hat{t}\|_2 \|\varphi\|_{\infty} \\ &\le \|N_n\|_2 \|t_n-t\|_1^{\frac{1}{2}}\|t_n-t\|_\infty^{\frac{1}{2}}\|\varphi\|_\infty\to 0.   
\end{aligned}
\end{equation*}
The second term converges to zero as consequence of the weak convergence of $N_n$ (\ref{convN}) and the fact that 
$\hat{t}\varphi \in L^2(\Omega)$. 

Therefore, up to a subsequence, $(J_n, t_n, N_n)$ satisfies (\ref{convt}), (\ref{convN}), (\ref{convJ}), (\ref{convJt}), and (\ref{convNt}). Moreover, $(\hat{J}, \hat{t}, \hat{N})\in\Lambda$. By the weak lower semicontinuity of the $L^2(\Omega)$-norm, and the lower semicontinuity of $\text{TV}(\Omega)$ in $L^1(\Omega)$, we obtain
\begin{equation*}
m\leq E(\hat{J},\hat{t},\hat{N})\leq \liminf_{n\to\infty} E(J_n, t_n, N_n) =m.
\end{equation*}
We conclude that $(\hat{J}, \hat{t},\hat{N})$ is a minimizer of $E(J,t,N)$.
\end{proof}

\subsection{Numerical approximation}
To minimize the proposed energy, we adopt a primal-dual strategy based on the Chambolle–Pock algorithm \cite{Chambolle}, which rewrites the minimization problem as a saddle-point formulation. 
The latter is solved through proximal updates associated with the primal and dual variables.  Since the data-fidelity term is nonconvex, its proximal step is approximated through a block-coordinate strategy, where each variable is updated while fixing the remaining ones at their latest estimates. This alternating minimization procedure leads to a local minimizer of the original energy \cite{chambolle2016introduction}. Moreover, all terms involving the residual component $N$ are smooth, allowing a closed-form minimization at each iteration.

In practice, the nonlocal weights are restricted to search windows around each pixel to reduce computational cost. Accordingly, the weights associated with the nonlocal gradient-type constraint are defined, for $l\in\{1,2,\ldots, d\}$, as
$$
\widehat{\omega}_l(x,y)=\frac{1}{\Gamma_l(x)}\exp\left(-\dfrac{\widehat{d}(\widehat{P}(x),\widehat{P}(y))}{\hat{h}^2_{\text{sim}}} \right)
$$
for $y\in B_{\hat{\nu}}(x)$, and set to zero otherwise. Here, $\hat{\nu}>0$ determines the size of the search window, $\widehat{P}(x)$ denotes a patch centered at $x$ extracted from the amplified gradient $V$ defined in \eqref{eq-v}, $\hat{h}_{\mathrm{sim}}>0$ is a filtering parameter, and $\Gamma_l(x)$ is a normalization factor. The similarity measure is computed as
$$
\widehat{d}(\widehat{P}(x),\widehat{P}(y))=\sum_{z\in B_{\hat{\kappa}}(0)} |V_l(x+z)-V_l(y+z)|^2,
$$
which corresponds to the squared Euclidean distance between color patches of size $(2\hat{\kappa}+1)\times (2\hat{\kappa}+1)$ centered at $x$ and $y$ on the $l$-th component of $V$. To avoid an excessive influence of the reference pixel,  $\widehat{w}(x,x)$ is set to the maximum weight within the search window for $y\neq x$. Finally, the weights of the nonlocal prior are defined analogously, taking into account both the spatial distance between pixels and the similarity between patches extracted from the degraded image $I$. 

\begin{figure}[t]
\centering
\subfloat[MambaResNet]{
\includegraphics[trim=0.5cm 0.25cm 1.1cm 0.25cm, clip, width=0.97\linewidth]{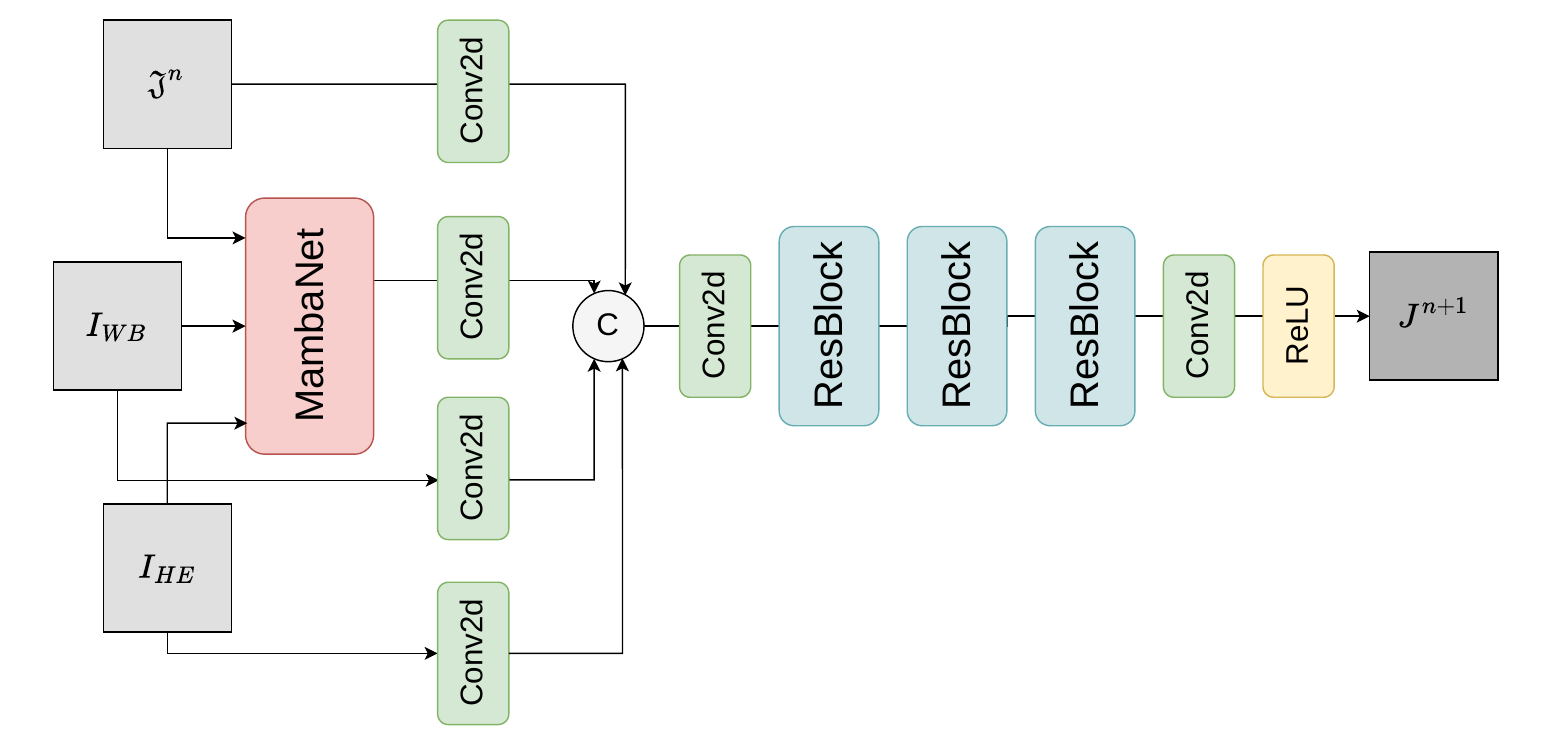}
}\\
\subfloat[MambaNet]{
\includegraphics[trim=0.5cm 0.25cm 0.7cm 0.25cm, clip, width=0.97\linewidth]{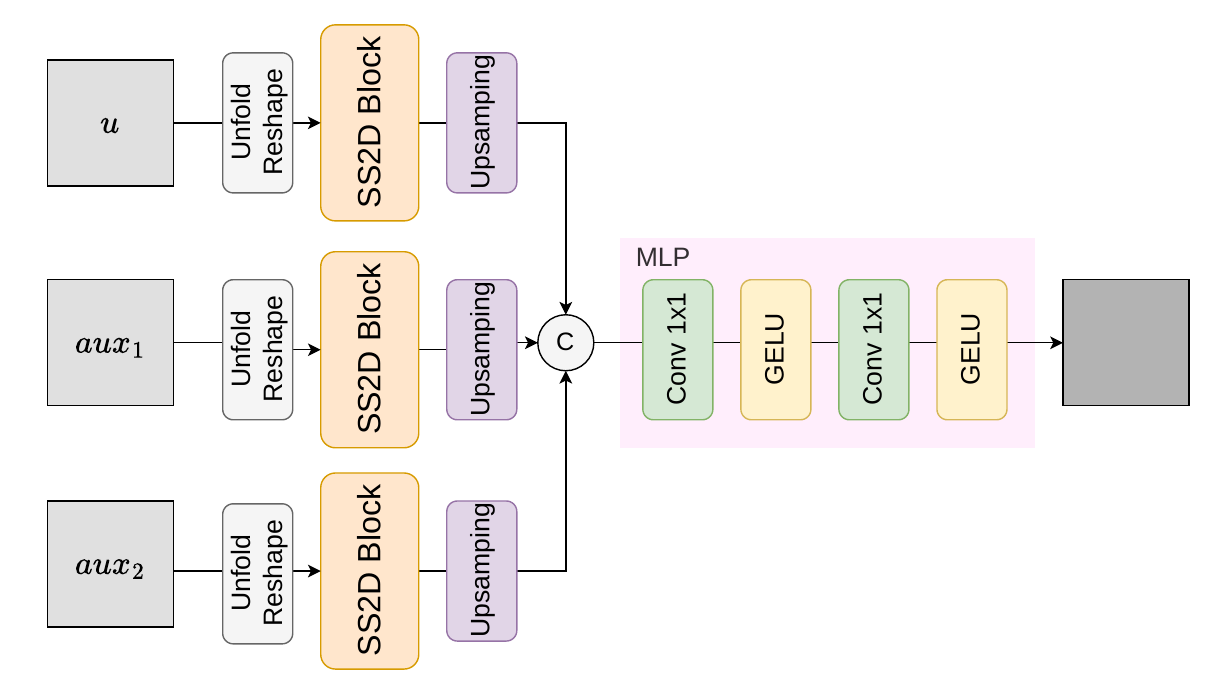}
}\\
\subfloat[SS2D Block]{
\includegraphics[trim=0.98cm 0.25cm 0.885cm 0.15cm, clip, width=0.97\linewidth]{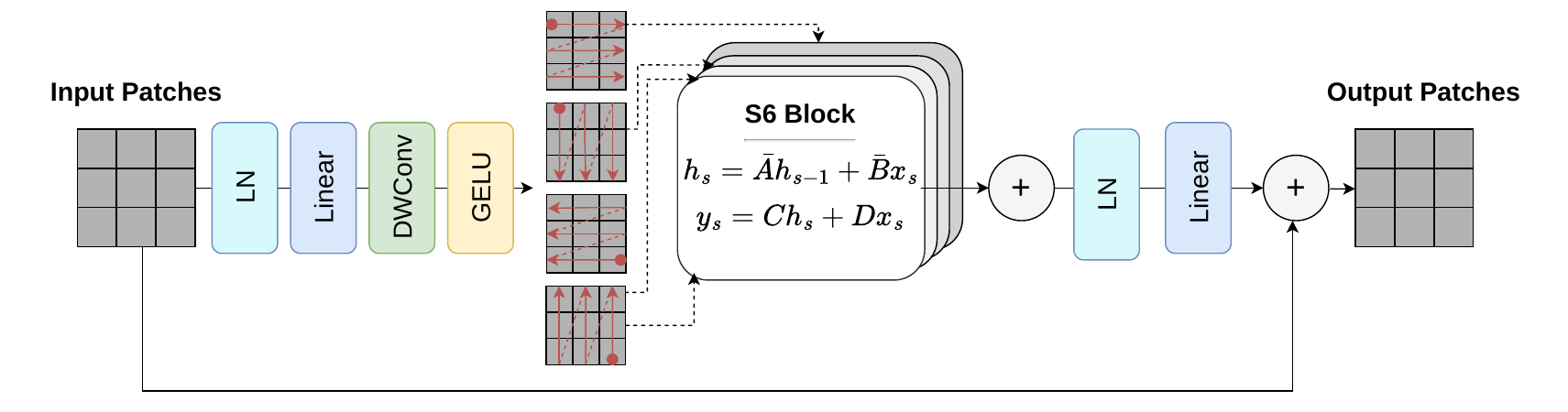}
}
\caption{Architectures of the networks involved in the proposed deep unfolding framework. (a) Proposed MambaResNet. (b) Proposed MambaNet module. When used as a component within MambaResNet, the inputs are $u=\mathfrak{J}^n, aux_1= I_{\text{WB}}$, and $aux_2= I_{\text{HE}}$. When employed to model the nonlocal operator in the gradient fidelity term, the inputs become $u=\nabla I, aux_1= \nabla I_{\text{WB}}$, and $aux_2= \nabla I_{\text{HE}}$. (c) SS2D block.}
\label{fig:architecture}
\end{figure}

\section{Deep Unfolding Network}

In this section, we develop a deep unfolding network based on the proposed variational model. The network is obtained by unfolding the iterations of the associated optimization algorithm and replacing the proximal operators with learning-based modules. This allows the model to learn image priors directly from data while preserving the interpretability of the underlying variational formulation.

\subsection{Algorithm unfolding}

We integrate two generic convex regularizers into the variational model \eqref{energia}, leading to the problem
\begin{equation*}\label{energia2}
\begin{aligned}
\min_{J,t,N}\; &\frac{1}{2}\|(J+N)t+A(1-t)-I\|_{2}^2+\alpha\Phi(J)+\beta \Psi(t)\\&+\frac{\lambda}{2} \|N\|_{2}^2 +\frac{\mu}{2} \|(\nabla J-V)_{\widehat{\omega}}\|_{2}^2+\frac{\rho}{2}\|t-t_0\|_2^2.
\end{aligned}
\end{equation*}

Following a block coordinate descent strategy \cite{chambolle2016introduction}, we iteratively minimize the energy with respect to one variable while keeping the others fixed at their latest estimates. The optimization steps are computed either through proximal gradient updates or in closed form when the objective is smooth, as in the case of $N$. This yields the following iterative scheme:

\begin{equation*}
    \begin{aligned}
    J^{n+1} &= \prox_{\tau \alpha \Phi} \left( \mathfrak{J}^{n} \right), \\
    t^{n+1}&=\prox_{\tau\beta\Psi}\left(\mathfrak{t}^{n}\right), \\
    N^{n+1} &= \dfrac{t^{n+1}\left(I - J^{n+1}t^{n+1}-(1-t^{n+1})A\right)}{\lambda+(t^{n+1})^2},
    \end{aligned}
\end{equation*}
where
\begin{equation*}
\begin{aligned}
&\mathfrak{J}_c^{n}(x)\hspace{-0.02cm} =\hspace{-0.02cm} J_c^n(x)\hspace{-0.05cm}  
-\hspace{-0.05cm} \tau t^n(x) \hspace{-0.02cm} \Big(\hspace{-0.05cm}(J_c^n(x)\hspace{-0.05cm} 
+\hspace{-0.05cm} N_c^n(x))t^n(x)\hspace{-0.05cm} 
-\hspace{-0.05cm} I_c(x)\Big.\\
&\Big. \hspace{-0.08cm} + A_c(x)(1\hspace{-0.05cm}-\hspace{-0.05cm} t^n(x))\hspace{-0.06cm}\Big)
\hspace{-0.08cm}+\hspace{-0.08cm} \tau \mu\, \div\hspace{-0.1cm} \left(\hspace{-0.1cm} 
\nabla J_c^n(x)\hspace{-0.01cm} - \hspace{-0.3cm}
\sum_{y\in B_{\hat{\nu}}(x)} \hspace{-0.2cm}
\widehat{\omega}(x,y)\hspace{-0.05cm} \circ \hspace{-0.05cm}V_c(y)\hspace{-0.1cm}\right)
\end{aligned}
\end{equation*}
and
\begin{equation*}\begin{aligned}
\mathfrak{t}^{n}= t^n&-\tau\rho(t^n-t_0)-\tau t^n\sum_{c=1}^C \left(J_c^{n+1}+N_c^n-A_c\right)^2\\
&-\tau\sum_{c=1}^C (A_c-I_c)\left(J_c^{n+1}+N_c^n-A_c\right).\end{aligned}\end{equation*}

By mapping the iterative scheme into a deep architecture, we obtain a network composed of $K$ stages. The proximal operator related to the transmission map regularization, $\prox_{\tau\beta\Psi}$, is replaced by a residual network, ProxNet$^n$. On the other hand, the proximal operator corresponding to the regularization of the sought image, $\prox_{\tau \alpha \Phi}$, is replaced by a Mamba-based residual network, MambaResNet$^n$, designed to mimic the behaviour of a nonlocal prior. The nonlocal operator involved in the gradient fidelity term is modeled through a multi-branch Mamba module, MambaNet$^n$. Therefore, the unfolded proximal gradient scheme becomes
\begin{equation*}
    \begin{aligned}
    J^{n+1} &= \text{MambaResNet}^{n+1}\left(\mathfrak{J}^{n} \right), \\
    t^{n+1}&=\text{ProxNet}^{n+1}\left(\mathfrak{t}^{n} \right), \\
    N^{n+1} &= \dfrac{t^{n+1}\left(I - J^{n+1}t^{n+1}-(1-t^{n+1})A\right)} {\lambda+ (t^{n+1})^2},
    \end{aligned}
\end{equation*}
where $\mathfrak{J}^{n}$ is now given by
\begin{equation*}\begin{aligned}
\mathfrak{J}^{n}= J^n &- \tau  t^n \left(\left(J^n + N^n\right) t^n - I  +A(1-t^n)\right)\\&+ \tau \mu \div \left( \nabla J^n - \text{MambaNet}^n\left(V\right) \right).\end{aligned}\end{equation*}
Thus, each stage consists of three successive updates corresponding to the three primal variables.

 The modules ProxNet$^n$, MambaResNet$^n$, and MambaNet$^n$ do not share learnable weights between the different stages, although the same architectures, which are described in the next subsection, are used throughout. The hyperparameters $\alpha, \beta, \lambda, \mu, \rho$, and $\tau$ are shared across all stages and learned during training. Moreover, $J$ is initialized as the degraded image, $N$ as an all-zero variable, and $t_0$, computed using Dark Channel Prior \cite{he2010single}, is used as the initial transmission map. Finally, we take $V$ as the image gradient $\nabla I$.

\begin{figure*}[t]
\centering
\begin{tabular}{c@{\hskip 0.05em} c@{\hskip 0.05em} c@{\hskip 0.05em} c@{\hskip 0.05em} c} 
    \spyimageStagestest{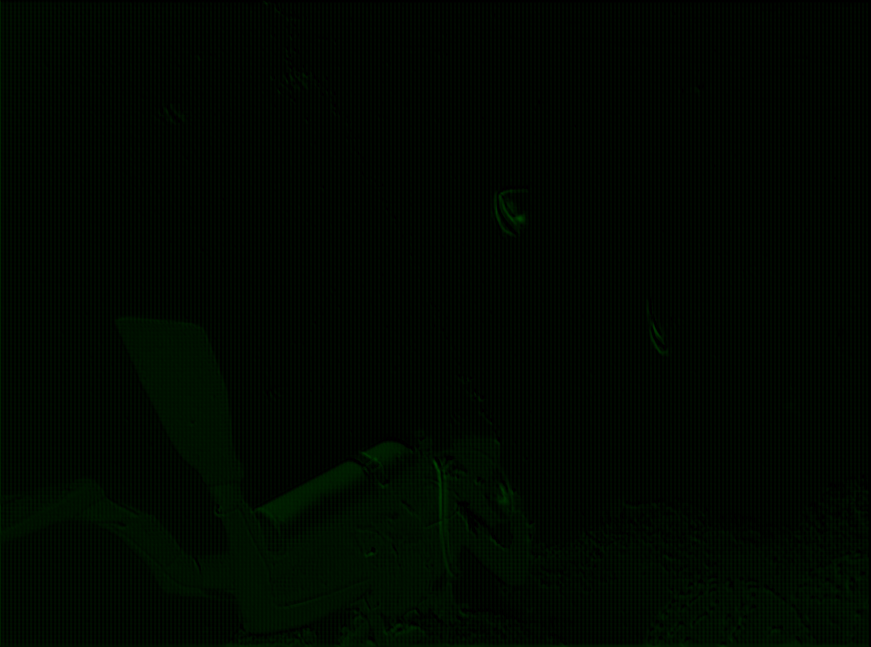}  & \spyimageStagestest{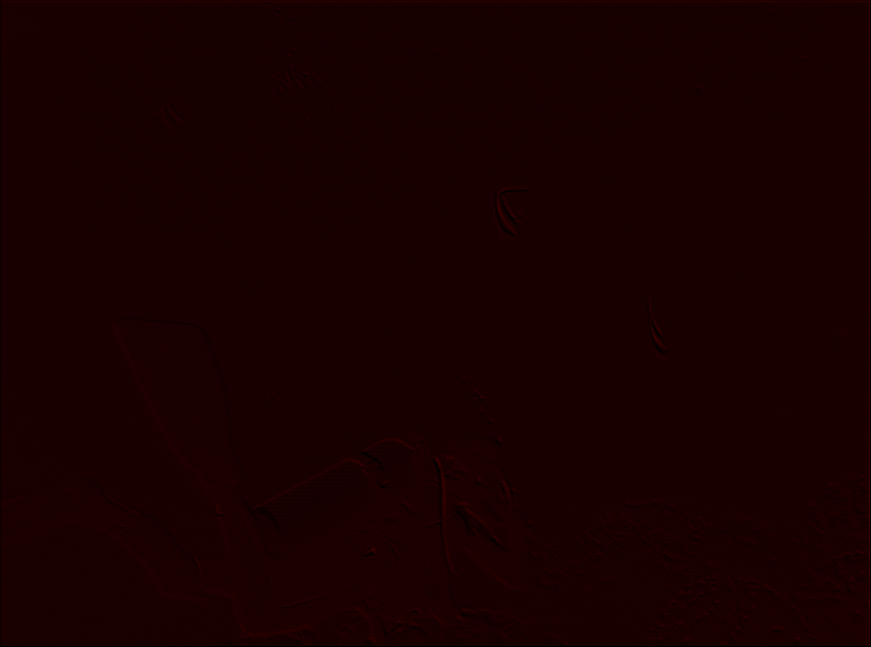}  & \spyimageStagestest{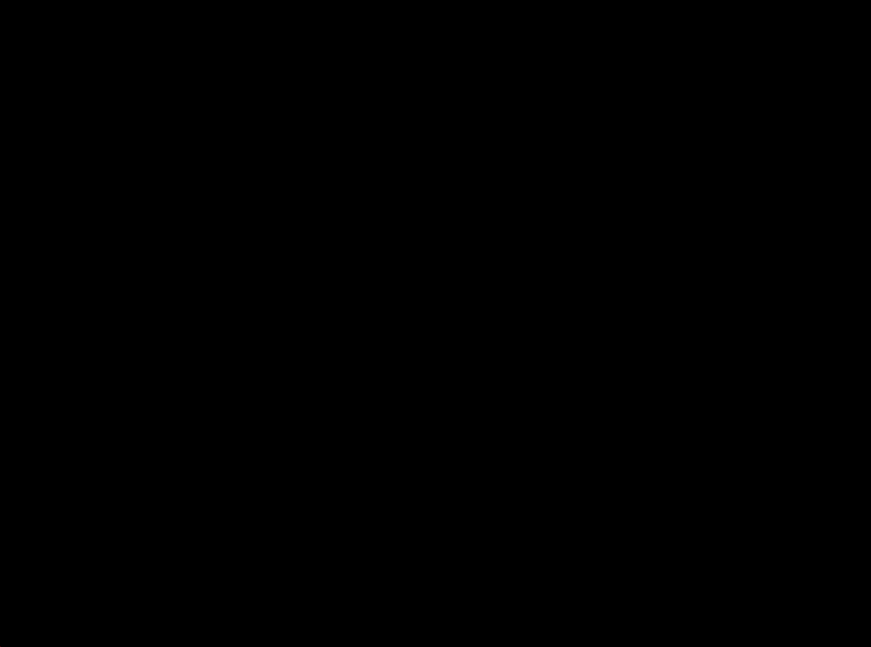}  &
    \spyimageStagestest{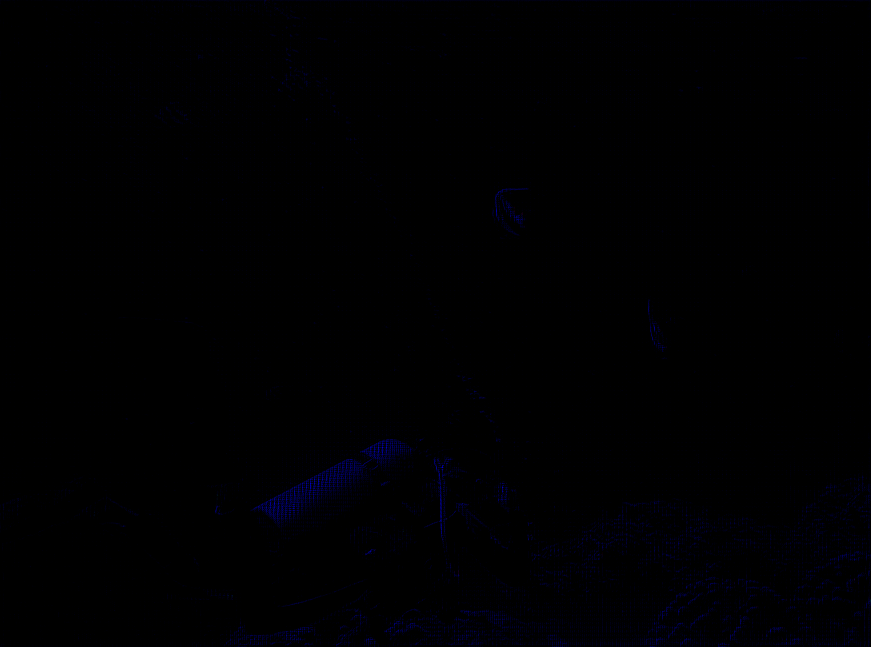}  &
    \spyimageStagestest{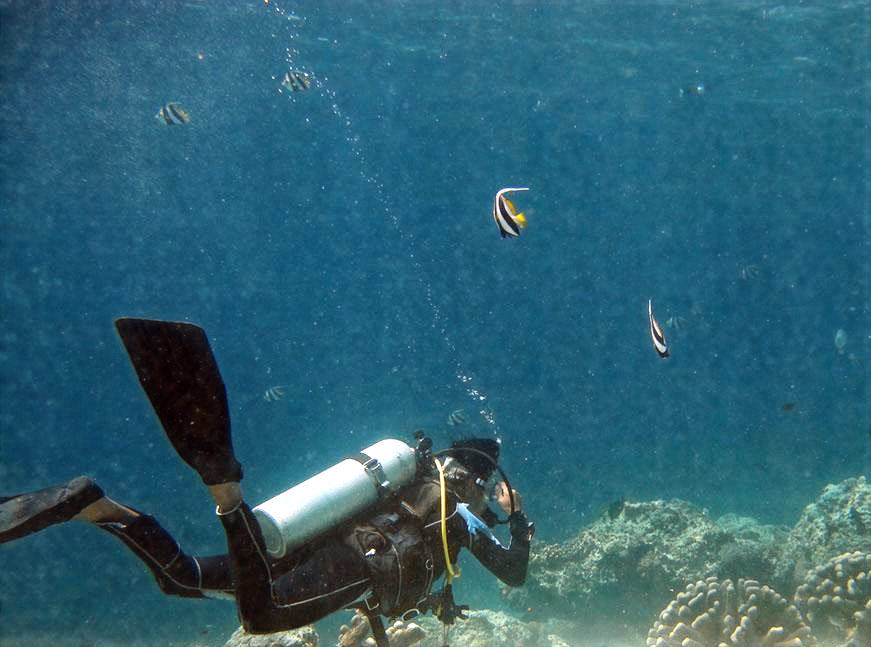} 
 \\
    \spyimageStagestest{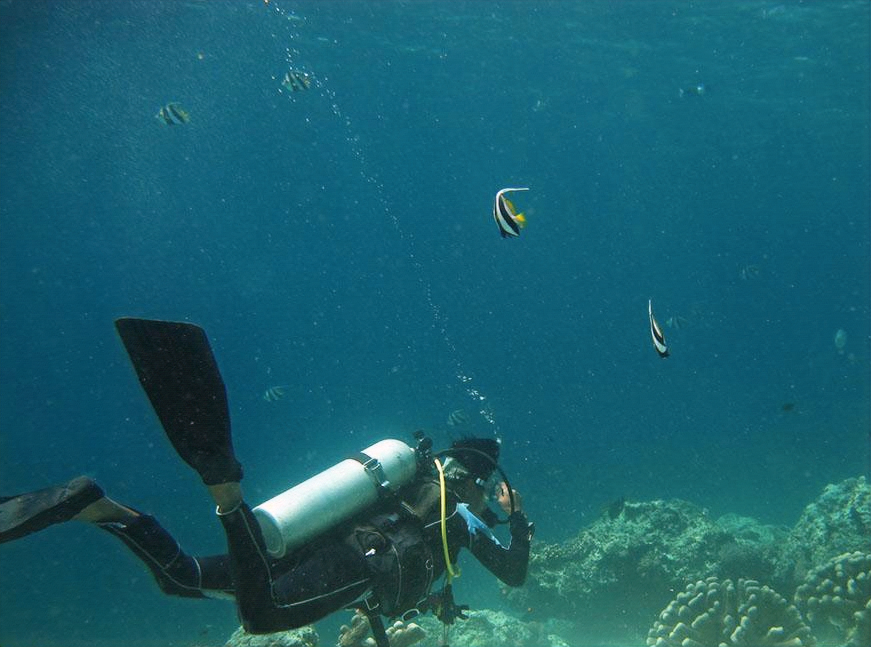}  & \spyimageStagestest{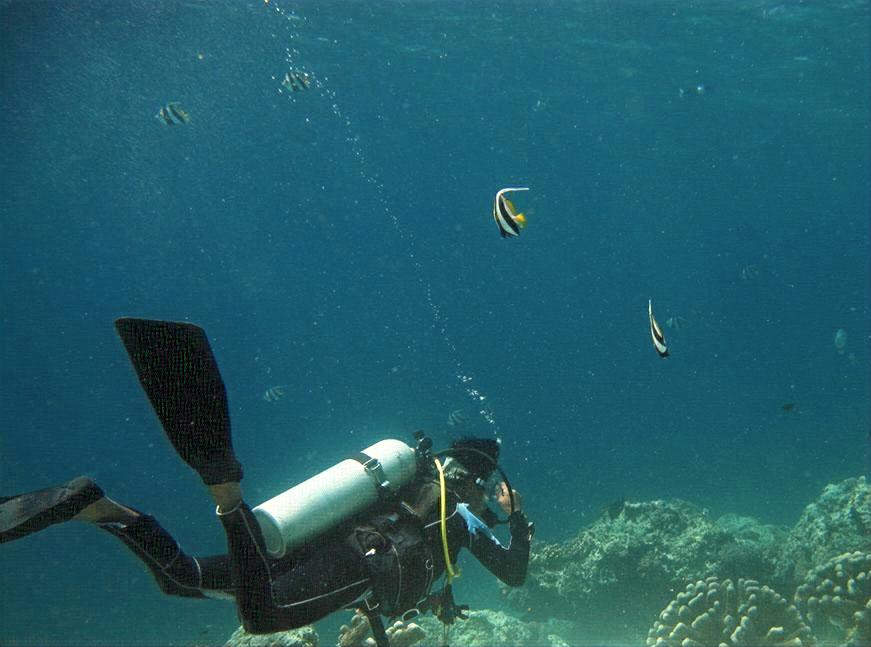}  & \spyimageStagestest{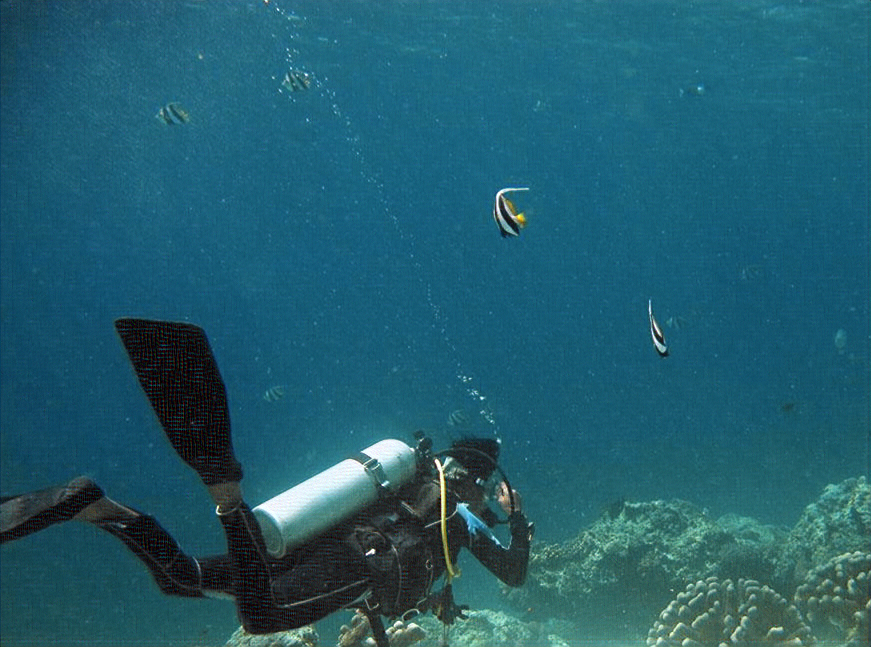}  &
    \spyimageStagestest{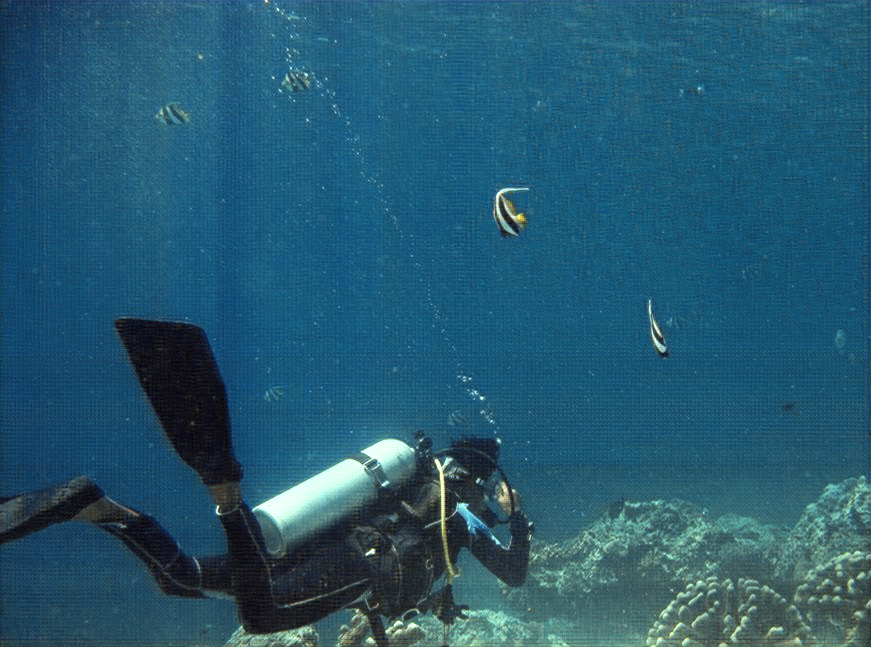}  &
    \spyimageStagestest{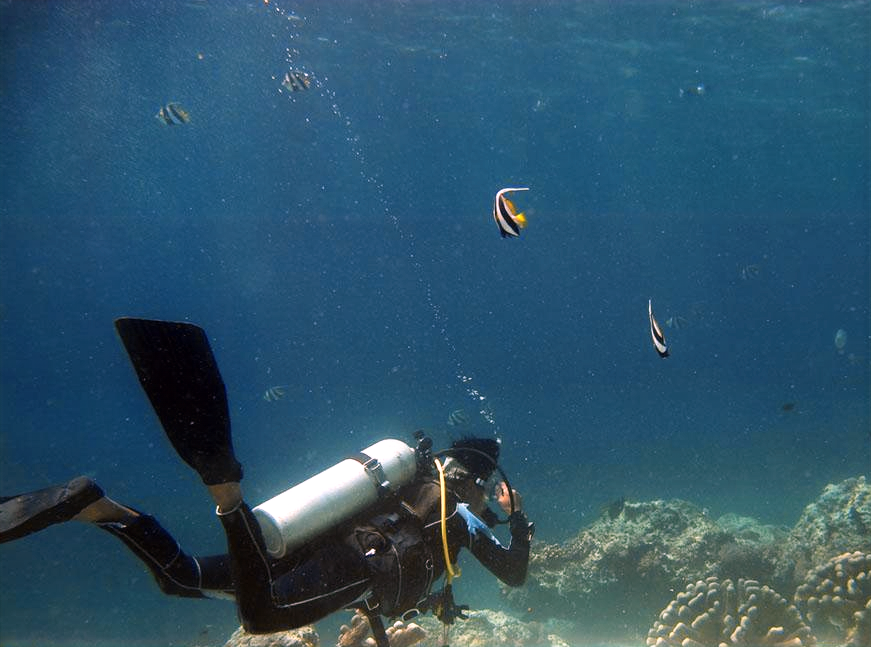} \\
    \spyimageStagestest{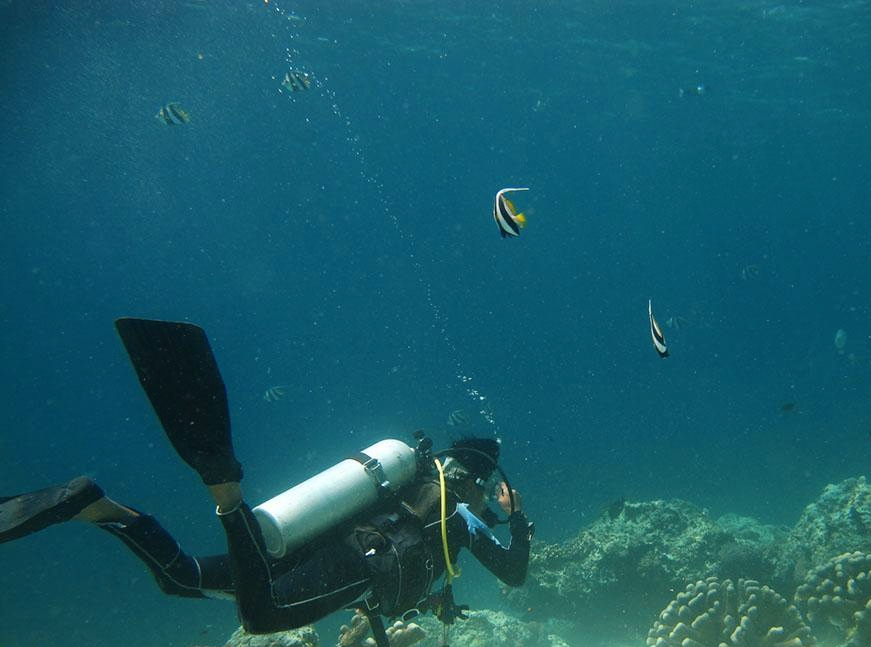}  & \spyimageStagestest{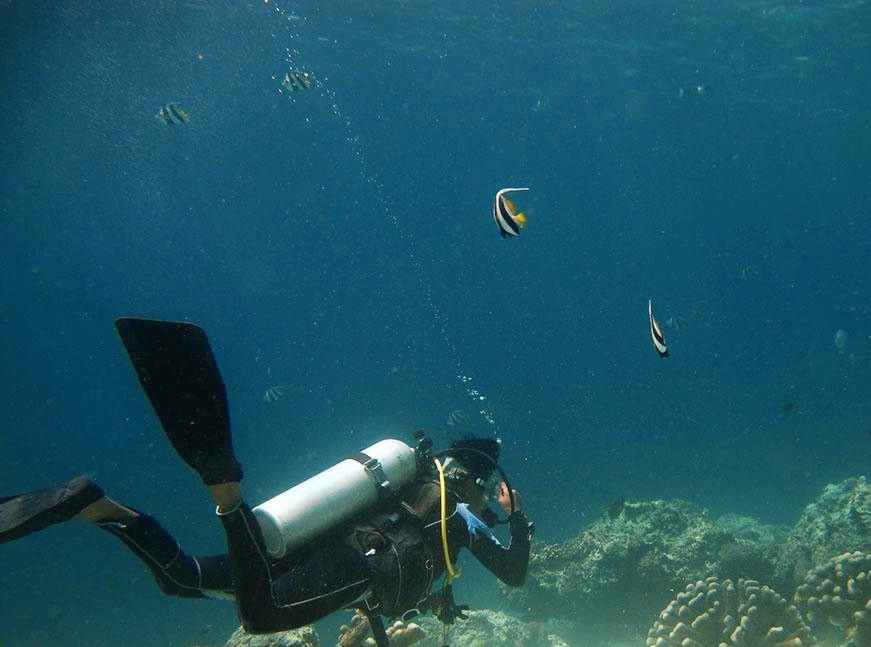}  & \spyimageStagestest{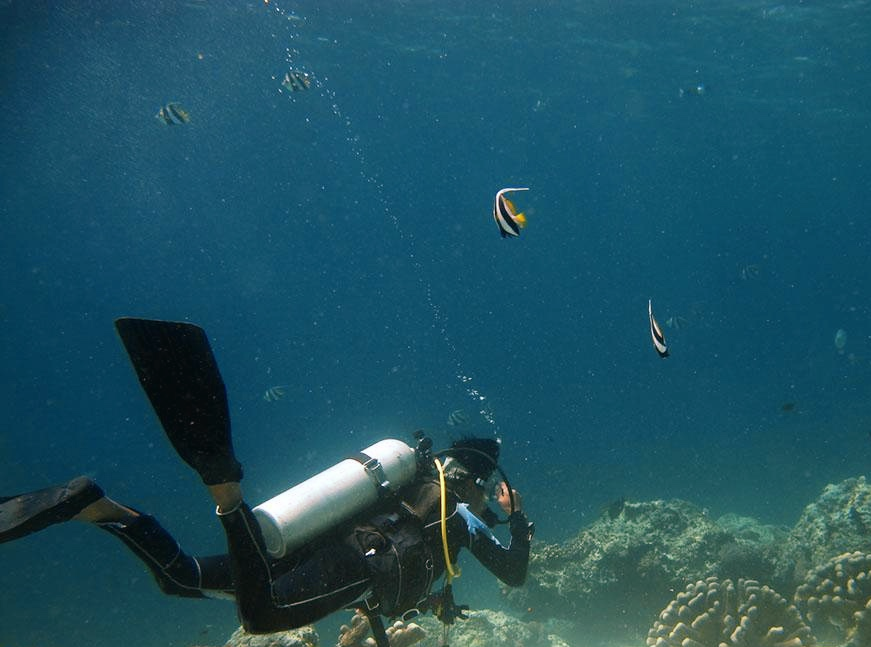}  &
    \spyimageStagestest{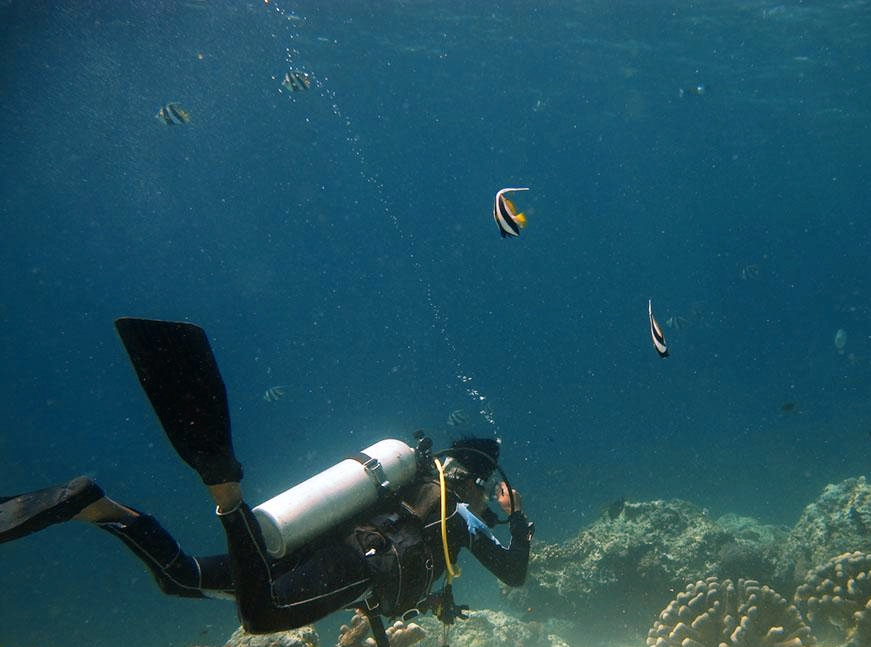}  &
    \spyimageStagestest{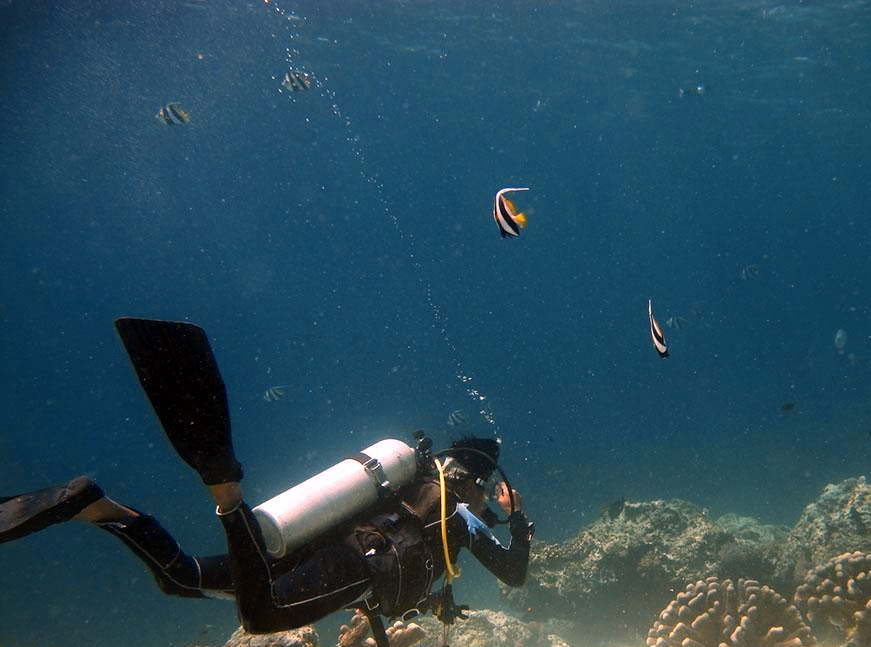}  \\
      Stage 1 &  Stage 2 &  Stage 3 & Stage 4 & Stage 5 \\
\end{tabular}
\caption{Comparison of intermediate stages of the unfolding network. First row: outputs obtained when only the MSE loss on the final output is used, showing that intermediate stages fail to produce meaningful results. Second row: outputs obtained with MSE loss and proximal trajectory supervision, illustrating that the intermediate stages progressively evolve toward the reference image. Third row: corresponding ground-truth target associated with each stage, ${J}_{\text{prox}}^{n}$.}
\label{fig:abl_stages}
\end{figure*}

\subsection{Architectures}\label{architecture}

In this section, we describe the architectures of ProxNet$^n$,  MambaResNet$^n$, and MambaNet$^n$, which are used to replace the proximal and nonlocal operators appearing in the unfolded optimization scheme.

Recall that the proximal operator satisfies
\begin{equation*}
{y}=\prox_{\tau\phi}(x)\iff {y}\in\left(Id+\tau\partial\phi\right)^{-1}\left(x\right),
\end{equation*}
which can be interpreted as a perturbation of the identity, making residual architectures particularly suitable for their modeling. Therefore, ProxNet$^n$, which replaces the proximal operator associated with the transmission map regularization, is implemented as a lightweight residual network. It consists of three residual blocks, each comprising two convolutional layers with skip connections, and ends with a sigmoid activation to enforce the physical constraint $t \in (0,1)$.

To efficiently capture long-range dependencies in the proximal operator associated with the image prior and in the nonlocal gradient-type constraint, we propose leveraging recent advances in State Space Models (SSMs) \cite{liu2024vmamba}. SSMs map an input sequence $x(s)\in\mathbb{R}$ to an output $y(s)\in \mathbb{R}$ through a hidden state $h(s)\in\mathbb{R}^{d_\text{state}}$. The continuous formulation is given by the linear system
\begin{equation*}
h'(s)=Ah(s)+Bx(s), \qquad y(s)=Ch(s)+Dx(s),
\end{equation*}
where $A,B,C,D$ are learnable parameters. To embed SSMs within the learning framework, the system is discretized with the Zero-Order Hold method using a step size $\Delta$, leading to
\begin{equation}\label{mamba_ssm}
h_s=\bar A h_{s-1}+\bar B x_s, \qquad y_s=Ch_s+Dx_s,
\end{equation}
where $\bar A=e^{\Delta A}$ and $\bar B=(\Delta A)^{-1}(e^{\Delta A}-I)\Delta B$.
While classical SSMs are linear time-invariant, Mamba \cite{gu2024mamba} introduces a selective SSM (S6) where the discretization parameters depend on the input. This preserves the recursive structure, enabling the modeling of long-range dependencies while remaining computationally efficient through a parallel scan algorithm.

MambaResNet$^n$ is a Mamba-based residual network, illustrated in Figure \ref{fig:architecture}(a). 
To better exploit complementary information derived from different enhancements of the degraded input, we adopt a multi-branch design inspired by image fusion strategies. Specifically, three parallel branches are employed. The first one receives the argument of the proximal operator, $\mathfrak{J}^n$, while the remaining two branches process auxiliary versions of the input image, namely, a white-balance enhanced image, $I_{\text{WB}}$, and a histogram-equalized image, $I_{\text{HE}}$.

Since SSMs operate on sequential data, the input images are first partitioned into non-overlapping patches using an unfolding operation. These patch tokens are then projected and processed by Mamba blocks employing a 2D Selective Scan (SS2D) mechanism \cite{liu2024vmamba}. The S6 module is extended to 2D vision data by scanning the tokenized image along four spatial directions (horizontal, vertical, and their reversed counterparts). The resulting features are subsequently aggregated to form a global representation. Figure \ref{fig:architecture}(c) shows a diagram of this SS2D block.

The outputs of the three Mamba branches are fused through a multi-layer perceptron (MLP), which adaptively combines the complementary representations extracted from the different inputs. 
Finally, a residual component is included in MambaResNet$^n$, ensuring consistency with the proximal interpretation. This design choice allows MambaResNet$^n$ to be a learnable approximation of the proximal operator associated with the nonlocal regularization of the clear image $J$.

Finally, a Mamba-based module, denoted as MambaNet$^n$, is employed to learn the nonlocal filtering operation associated with the gradient fidelity term. The dependence on the weighting function $\hat{\omega}$ is implicitly captured through the input-dependent parameters of the Mamba model in \eqref{mamba_ssm}.  In contrast to MambaResNet$^n$, which operates on image intensities, MambaNet$^n$ processes gradient-domain inputs, enabling the modeling of nonlocal interactions between image gradients. Specifically, the network takes $V=\nabla I$ as the primary input, while the auxiliary branches process the gradients of the enhanced and histogram-equalized images, $\nabla I_{\text{WB}}$ and $\nabla I_{\text{HE}}$. In this case, the output of the MLP is directly used as the enhanced gradient representation. The resulting architecture is illustrated in Figure \ref{fig:architecture}(b).




\subsection{Proximal trajectory loss}

The flexibility and interpretability of unfolding networks originate from their connection to proximal optimization algorithms. However, these same properties also introduce specific challenges. In particular, interpretability depends on the extent to which the learned neural modules behave as true proximal operators. In practice, training objectives typically enforce only that the final output is close to the ground truth, without constraining intermediate stages. As a result, the intermediate iterates of the unfolding network may deviate from the expected optimization trajectory, as observed in Figure~\ref{fig:abl_stages}.

To address this issue, Wang et al.~\cite{wang2025proximal} propose a proximal trajectory loss that encourages the network to follow an explicit optimization path by learning an ideal restoration regularizer. In the underwater imaging context, we adopt a similar principle and define the ideal regularizer by solving the following optimization problem:
\begin{equation*}\min_J\frac{1}{2}\|(J+N)t+A(1-t)-I\|_2^2+\frac{\vartheta}{2}\|J-J_{\text{gt}}\|_2^2,\end{equation*}
where $J_{\text{gt}}$ represents the ground-truth image. Applying a proximal gradient step, where the first term is treated explicitly and the second one through its proximal operator, yields the following iterative scheme:
\begin{equation}
\label{primal_dual_gt}
    \begin{array}{l}
    J^{n+1}=\frac{J^n-\tau t\left((J^n+N)t+A(1-t)-I\right)+\tau\vartheta J_{\text{gt}}}{1+\tau\vartheta}.    \end{array}
\end{equation}

\begin{figure*}[p]
    \centering
\begin{tabular}{c@{\hskip 0.2em} c@{\hskip 0.2em} c@{\hskip 0.2em} c}
\spyimageUIEB{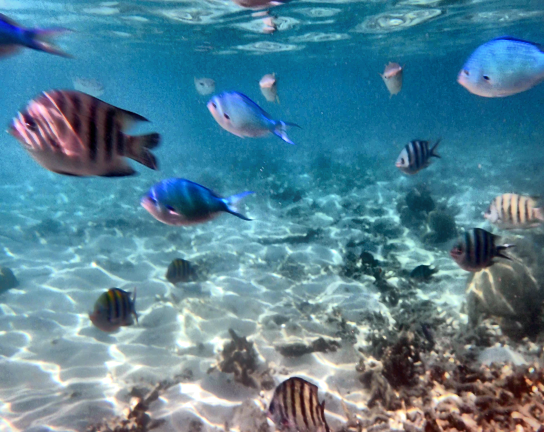} &\spyimageUIEB{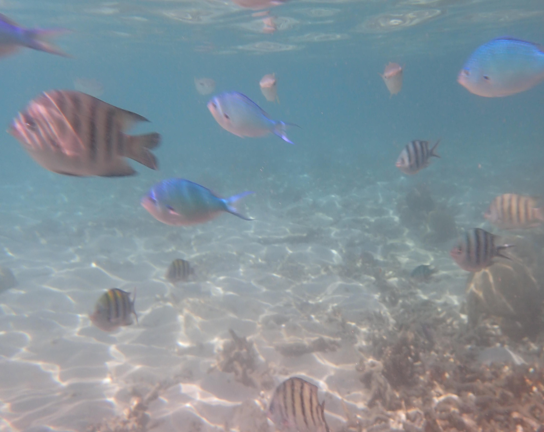}
    &
    \spyimageUIEB{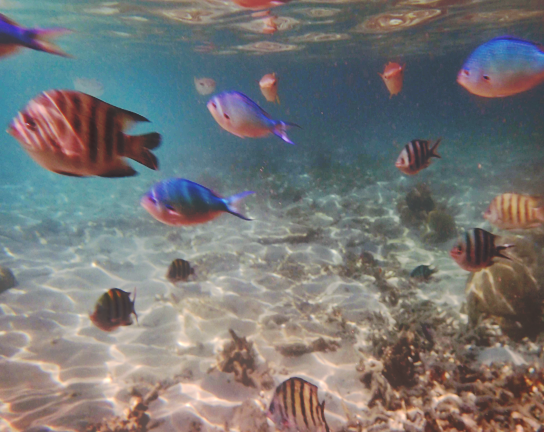} &
    \spyimageUIEB{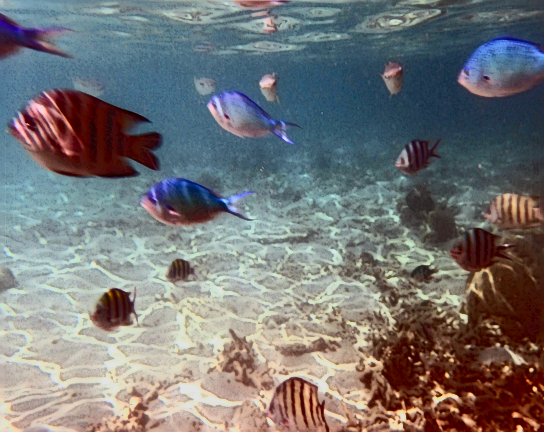} \\
     Ground Truth & Input &  CBAF  \cite{ancuti2017} &  BayesianRetinex \cite{zhuang2021bayesian} \\ 
    \spyimageUIEB{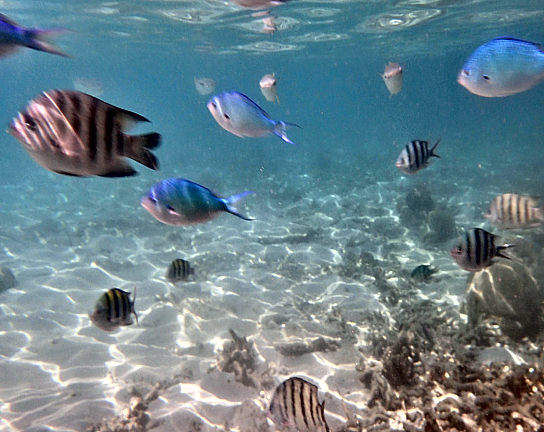} &\spyimageUIEB{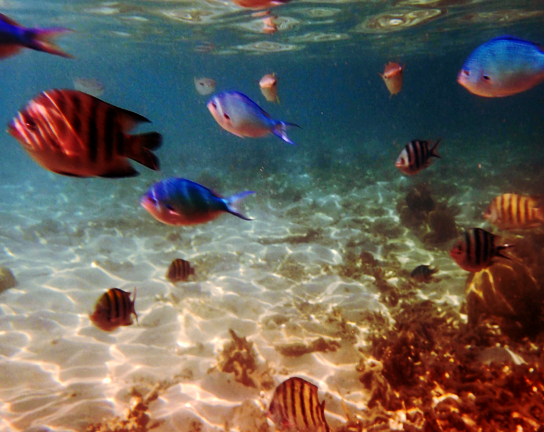}
    &
    \spyimageUIEB{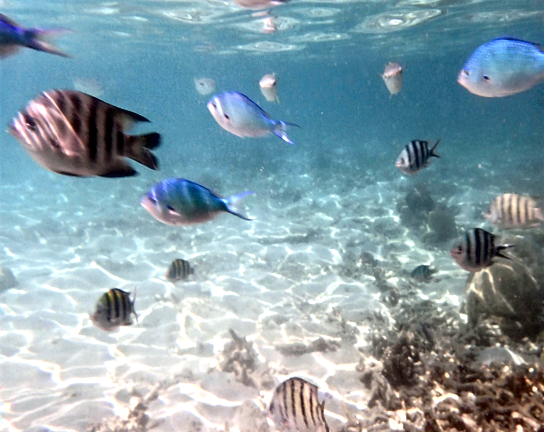} &
    \spyimageUIEB{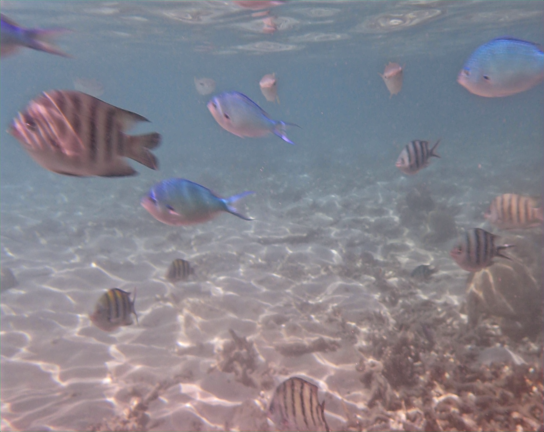} \\
    MMLE \cite{zhang2022underwater} & HFM \cite{an2024hfm} & WWPF \cite{zhang2024underwater} &   Shallow-UWNet \cite{naik2021shallow} \\ \spyimageUIEB{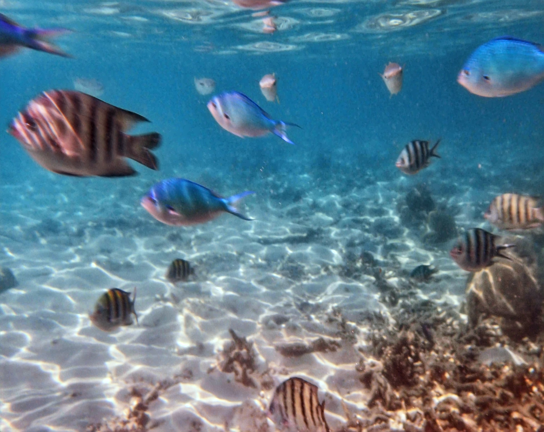} &\spyimageUIEB{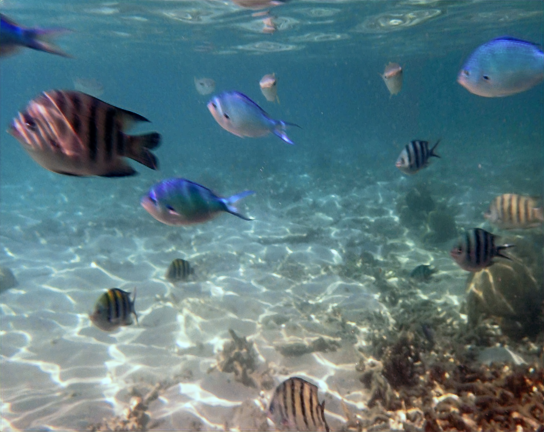}
    &
    \spyimageUIEB{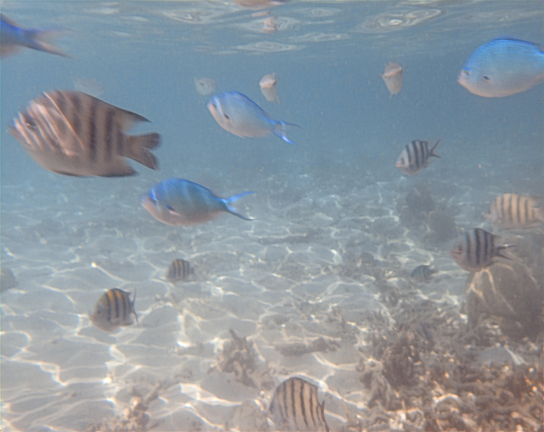} &
    \spyimageUIEB{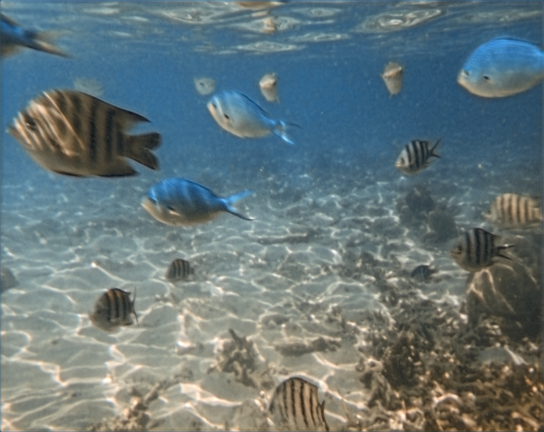} \\
    UColor \cite{li2021underwater} &   UIEC\^{}2-Net\cite{wang2021uiec} &  AutoEnhancer \cite{tang2022autoenhancer} &     UnfoldNet \cite{pham2023deep} \\ 
      \spyimageUIEB{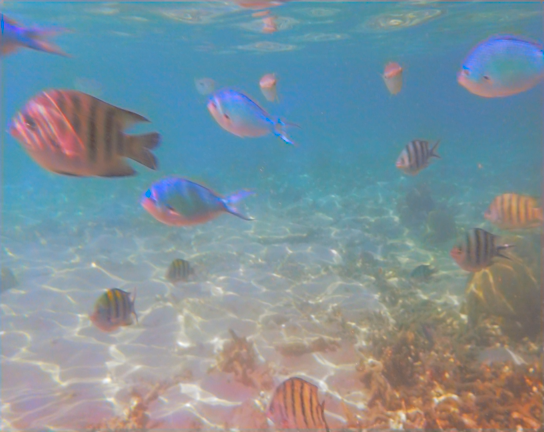} &
    \spyimageUIEB{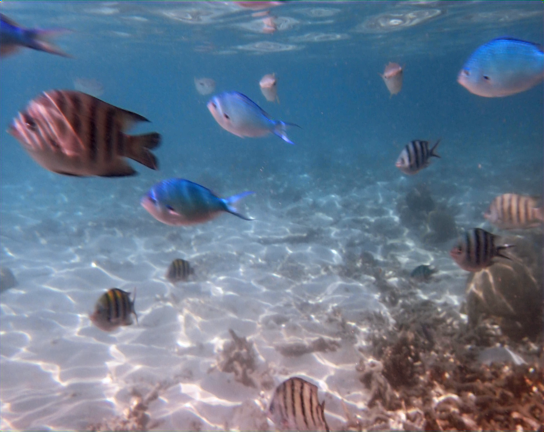} &
    \spyimageUIEB{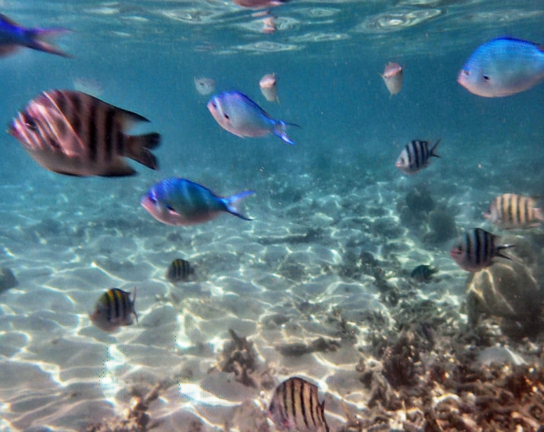} &
    \spyimageUIEB{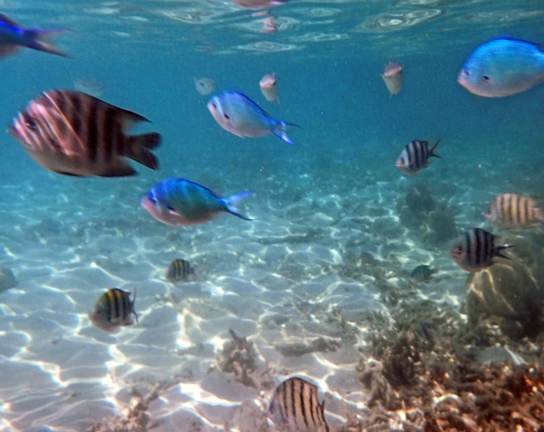} \\
 P2CNet \cite{rao2023deep} & LiteEnhanceNet \cite{zhang2024liteenhancenet} &  RAUNE-Net\cite{peng2023raune} & Mamba-UIE \cite{zhang2024mamba}\\
        \spyimageUIEB{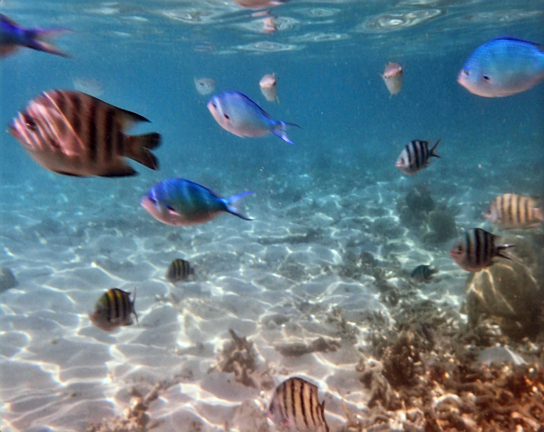} &
     \spyimageUIEB{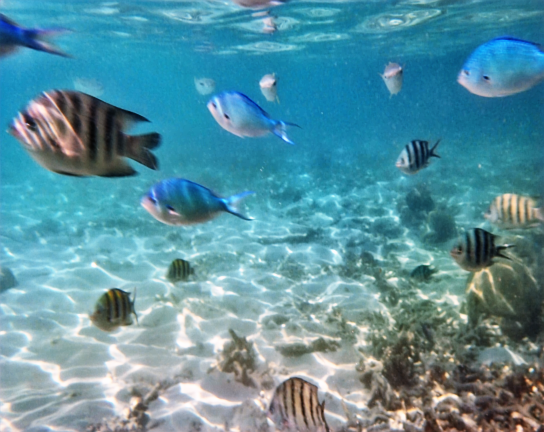} &
    \spyimageUIEB{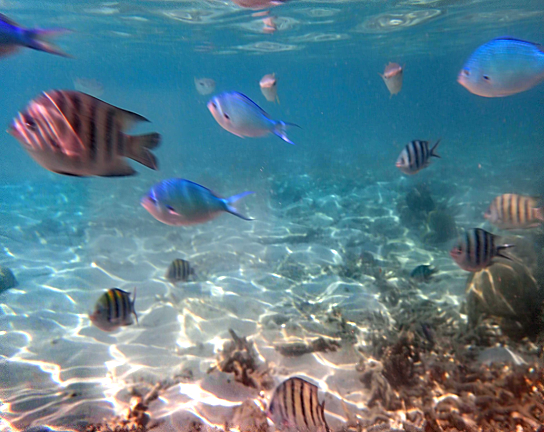} &
    \spyimageUIEB{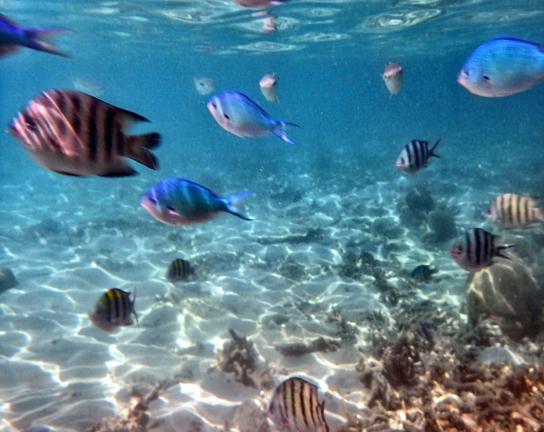} \\
     Phaseformer \cite{khan2025phaseformer}   & UIR-PolyKernel \cite{guo2025underwater} & Ours-Variational Model & Ours-Unfolding \\
\end{tabular}
    \caption{Visual comparison on a cropped image from the UIEB test set \cite{li2019underwater}. Several methods produce noticeable color distortions or incomplete haze removal, resulting in reddish or greenish tones in the enhanced images. Among the best-performing methods, UColor and Phaseformer introduce excessive reddish tones, while Mamba-UIE and RAUNE-Net are less effective at preserving fine details in the zoomed area and still exhibit a slight greenish tone in the background compared with our unfolding approach.}
    \label{fig:UIEBcomparison}
\end{figure*}

\begin{figure*}[p]
    \centering
\begin{tabular}{c@{\hskip 0.2em} c@{\hskip 0.2em} c@{\hskip 0.2em} c}
\spyimageSUIME{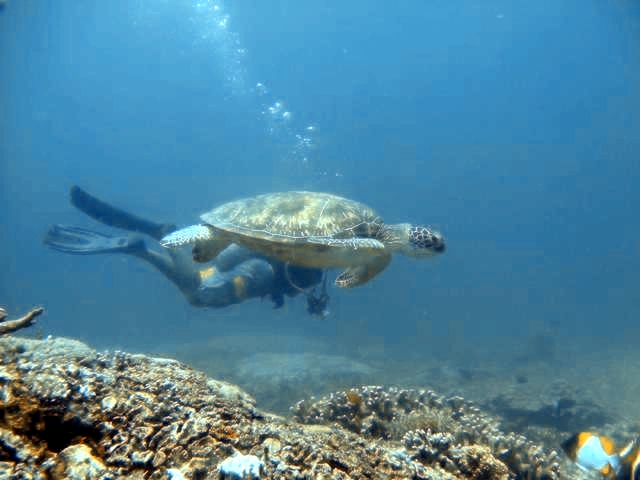} &\spyimageSUIME{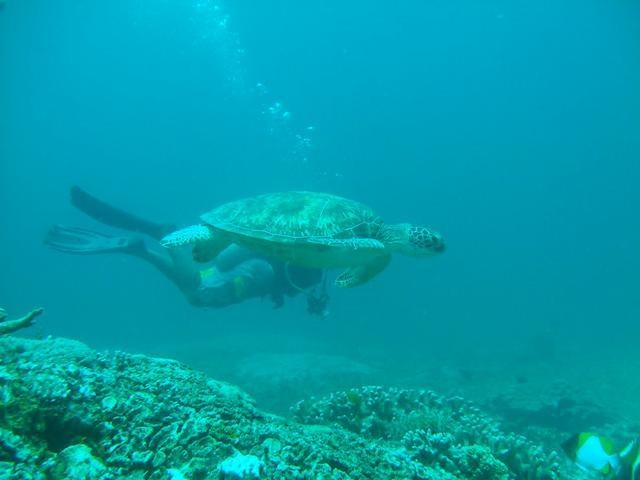}
    &
    \spyimageSUIME{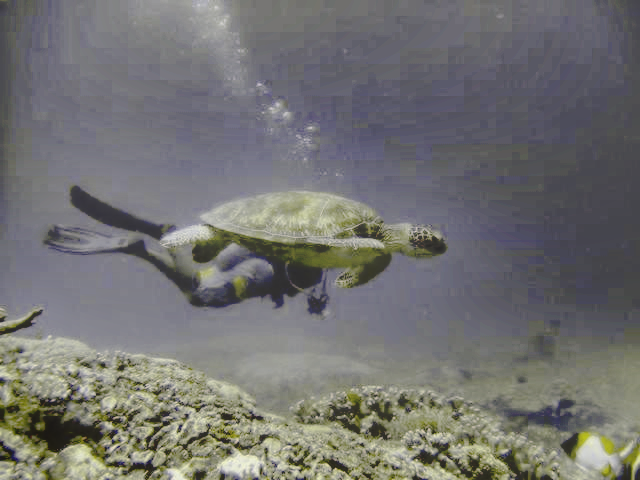} &
    \spyimageSUIME{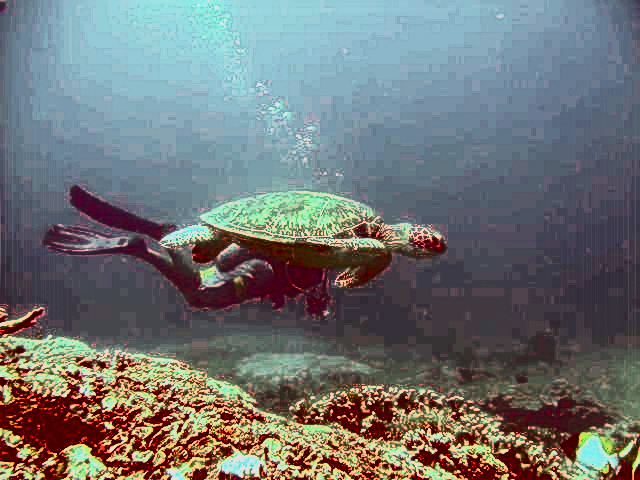} \\
     Ground Truth & Input &  CBAF  \cite{ancuti2017} &  BayesianRetinex \cite{zhuang2021bayesian} \\ 
    \spyimageSUIME{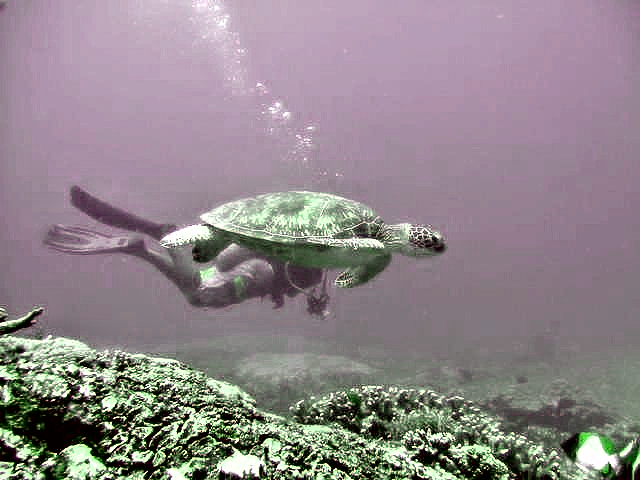} &\spyimageSUIME{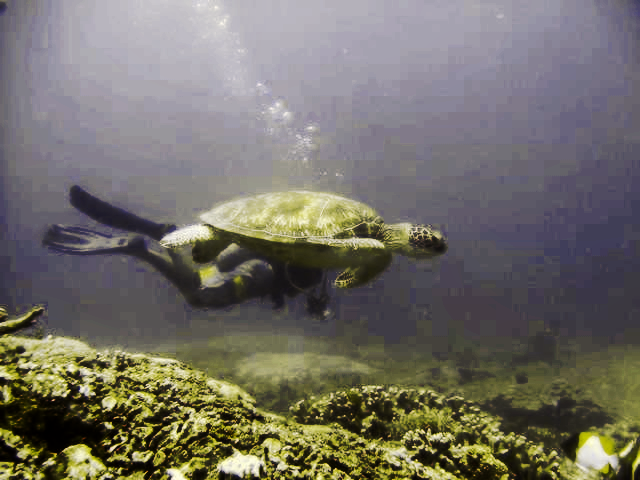}
    &
    \spyimageSUIME{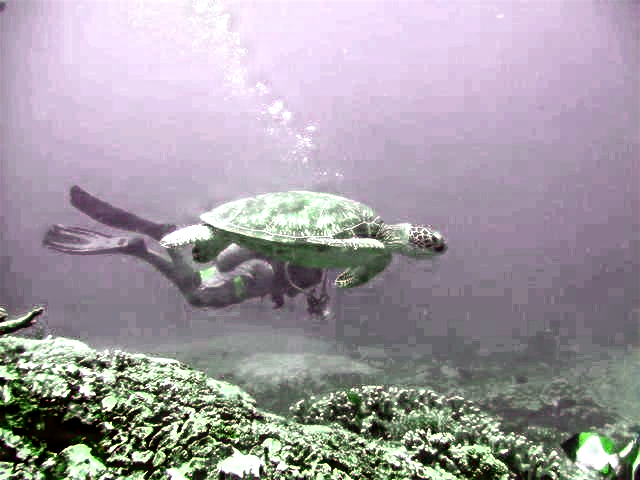} &
    \spyimageSUIME{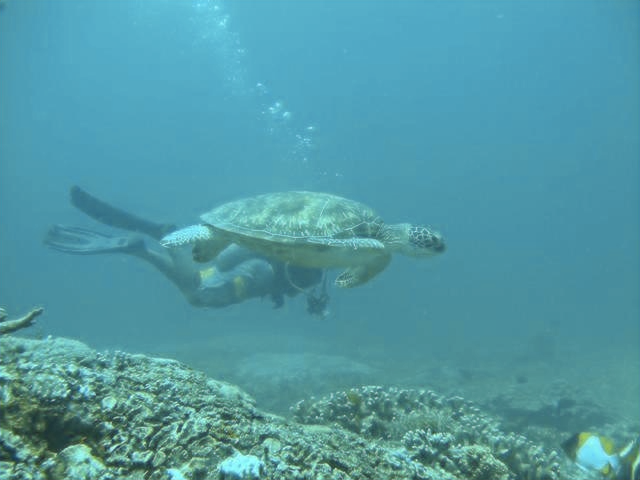} \\
    MMLE \cite{zhang2022underwater} & HFM \cite{an2024hfm} & WWPF \cite{zhang2024underwater} &   Shallow-UWNet \cite{naik2021shallow} \\ \spyimageSUIME{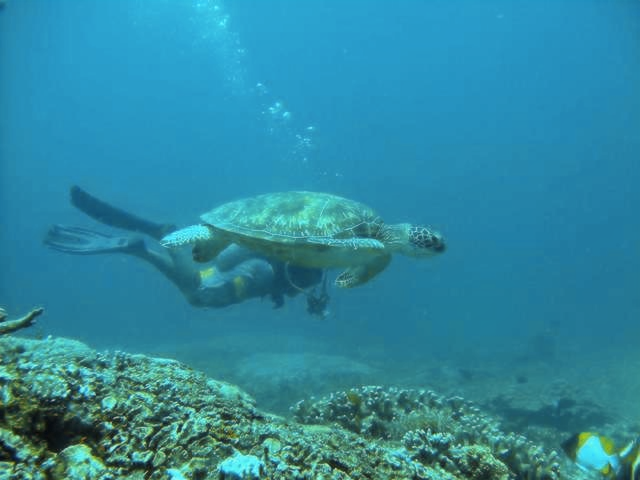} &\spyimageSUIME{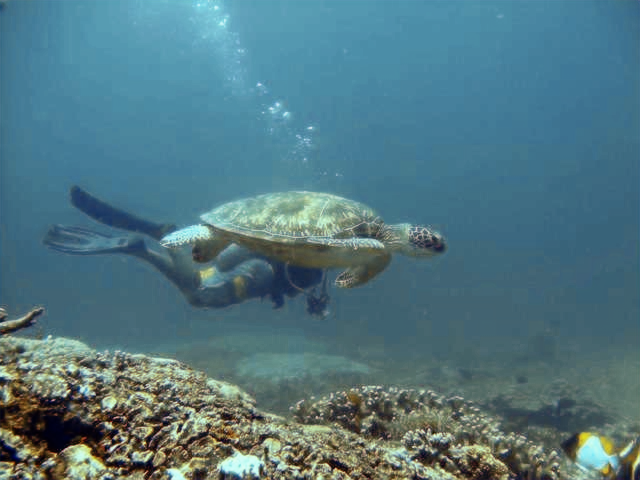}
    &
    \spyimageSUIME{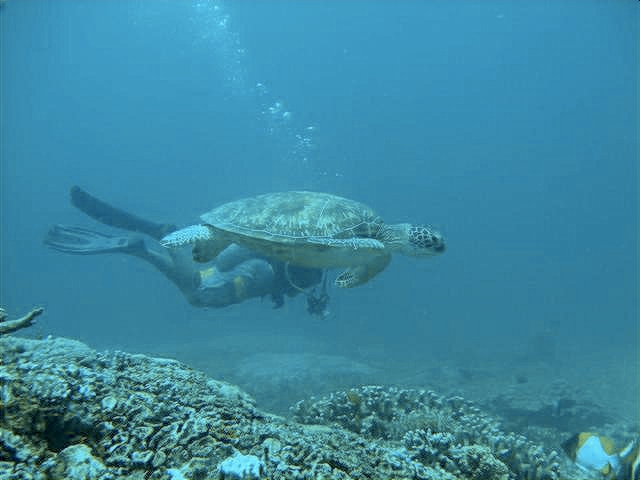} &
    \spyimageSUIME{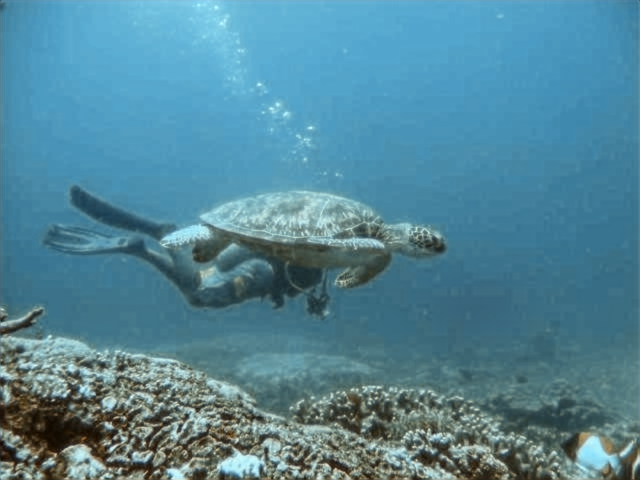} \\
    UColor \cite{li2021underwater} &   UIEC\^{}2-Net\cite{wang2021uiec} &  AutoEnhancer \cite{tang2022autoenhancer} &     UnfoldNet \cite{pham2023deep} \\ 
      \spyimageSUIME{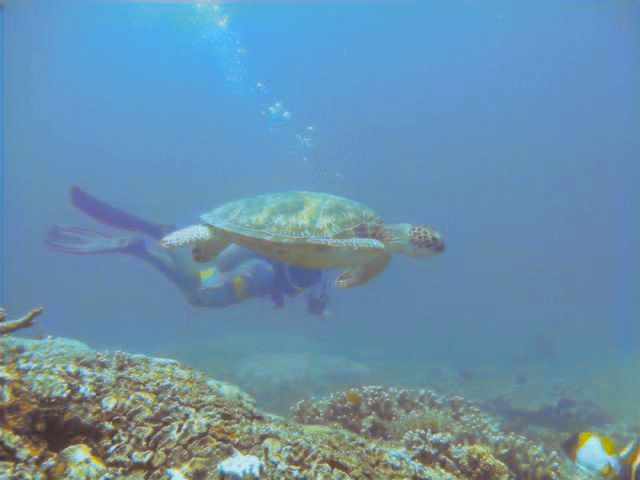} &
    \spyimageSUIME{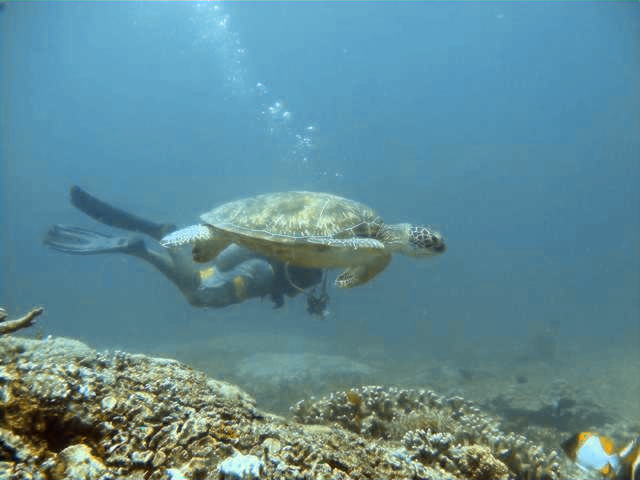} &
    \spyimageSUIME{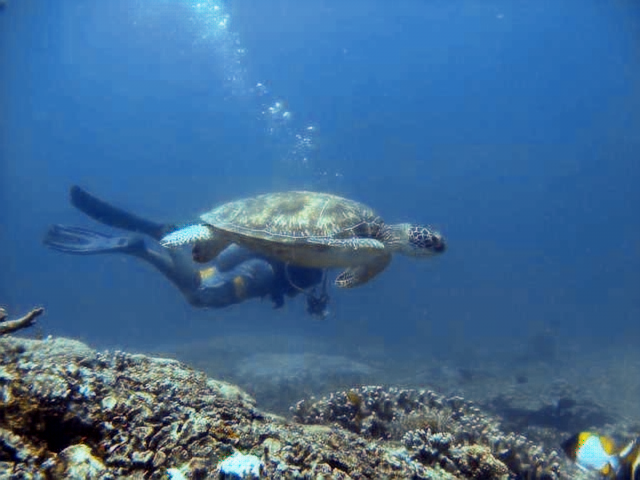} &
    \spyimageSUIME{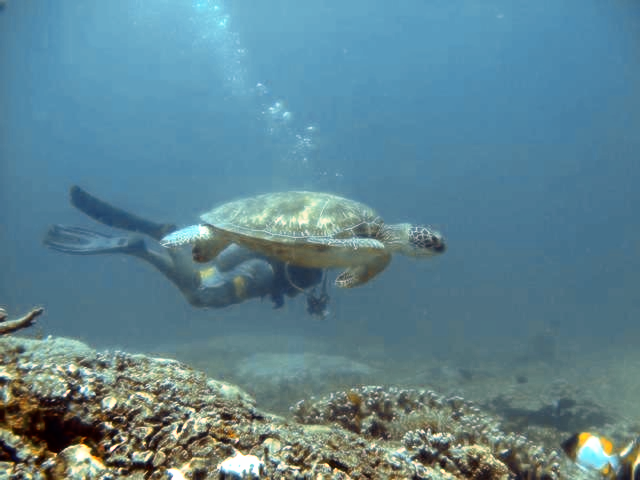} \\
 P2CNet \cite{rao2023deep} & LiteEnhanceNet \cite{zhang2024liteenhancenet} &  RAUNE-Net\cite{peng2023raune} & Mamba-UIE \cite{zhang2024mamba}\\
        \spyimageSUIME{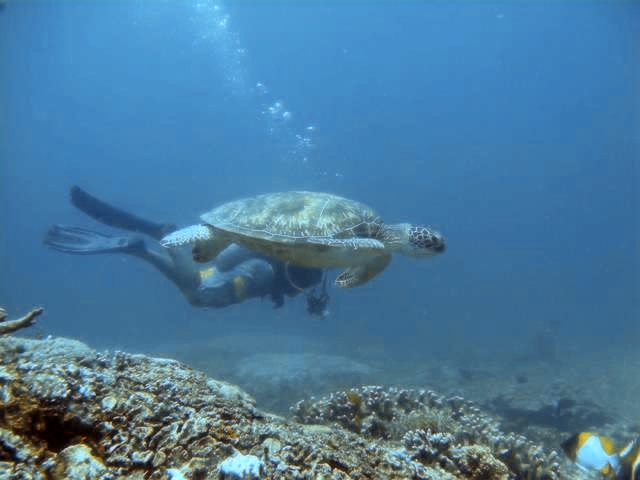} &
     \spyimageSUIME{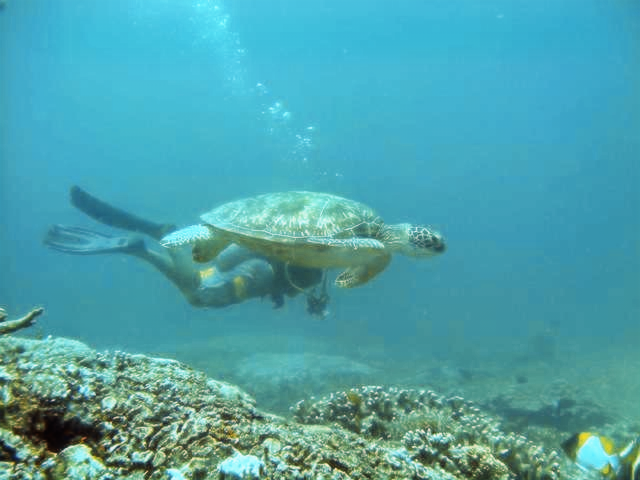} &
    \spyimageSUIME{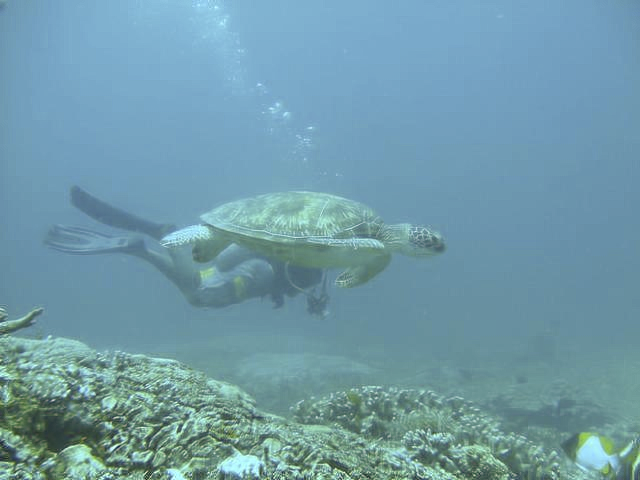} &
    \spyimageSUIME{ours_U.png} \\
     Phaseformer \cite{khan2025phaseformer}   & UIR-PolyKernel \cite{guo2025underwater} & Ours-Variational Model & Ours-Unfolding \\
\end{tabular}
    \caption{Visual comparison on a cropped image from the SUIM-E test set \cite{qi2022sguie}. Classical approaches produce unnatural water colors, while several deep-learning methods retain greenish tones. LiteEnhanceNet, Phaseformer, Mamba-UIE, and our unfolding approach achieve the best visual quality. However, the proposed unfolding approach more accurately recovers the orange and white regions of the fish in the zoomed area. }
    \label{fig:SUIMEcomparison}
\end{figure*}

Moreover, for computational efficiency, unfolding networks are typically implemented with a limited number of stages $K$, which is smaller than the number of iterations required for convergence of the corresponding classical optimization algorithm. Assuming that the classical solver converges in $S$ iterations, we align the $n$-th stage of the unfolding network with the $\frac{n \cdot S}{K}$-th iteration of the classical algorithm. This leads to the following proximal trajectory loss:
\begin{equation*}
\mathcal{L}_{\prox}(\{J_{\text{out}}^n\}_{n=1}^{K-1}, \{J_{\text{prox}}^{\frac{n\cdot S}{K}}\}_{n=1}^{K-1})=\sum_{n=1}^{K-1}\alpha_n\|J_{\text{out}}^{n}-{J}_{\text{prox}}^{\frac{n\cdot S}{K}}\|_2^2,
\end{equation*}
where $J^n_{\text{out}}$ is the output of the network at stage $n$, ${J}_{\text{prox}}^{n}$ denotes the $n$-th iterate obtained from \eqref{primal_dual_gt}, and the coefficients $\{\alpha_n\}$ balance the weight of each stage.

\begin{table}[t]
\centering
\begin{tabular}{lccc} Method   & PSNR $\uparrow$ & SSIM $\uparrow$ &  LPIPS $\downarrow$\\ \hline Classical  \\ \hline 

 CBAF \cite{ancuti2017} &  16.99&  0.8028&  0.242
\\ 
 BayesianRetinex \cite{zhuang2021bayesian} &  16.31&  0.7350&  0.299
\\ 
 UNTV \cite{xie2022variational} &  15.46&  0.6044&  0.332
\\ 
 MMLE \cite{zhang2022underwater} &  17.66&  0.7408&  0.237
\\ 
 PCDE \cite{zhang2023underwater} &  14.49&  0.6220&  0.348
\\ 
 HFM \cite{an2024hfm} &  16.15&  0.7768&  0.232
\\ 
WWPF \cite{zhang2024underwater} &  18.16&  0.7742&  0.192
\\
Variational Model (Ours) &  22.61&  0.8839 &  0.140
\\
\hline       Deep learning \\ \hline Shallow-UWNet \cite{naik2021shallow} &  17.70&  0.8014&  0.194
\\ 
UColor \cite{li2021underwater} &  22.40&  0.8880&  0.135
\\ 
 UIEC\^{}2-Net\cite{wang2021uiec} &  21.17&  0.8735&  0.115
\\ 
 AutoEnhancer \cite{tang2022autoenhancer} &  18.71&  0.8234&  0.182
\\ 
 UnfoldNet \cite{pham2023deep} &  20.90&  0.8567&  0.201
\\ 
 P2CNet \cite{rao2023deep} &  17.32&  0.7334&  0.227
\\ 
 LiteEnhanceNet \cite{zhang2024liteenhancenet} &  20.88&  0.8567&  0.141
\\ 
RAUNE-Net\cite{peng2023raune} &  22.20&  0.8315&  0.134
\\ 
  Mamba-UIE \cite{zhang2024mamba} &  \underline{23.27} &   \underline{0.9031} &  \underline{0.108}
\\ 
Phaseformer \cite{khan2025phaseformer} &  22.76&  0.8252&  0.130
\\ 
 UIR-PolyKernel \cite{guo2025underwater} &  21.88&  0.8948&  0.120\\ \hline
  Ours &  \bf{25.16} &  \bf{0.9140} &  \bf{0.089} 
\\ 
\end{tabular}
\caption{Quantitative evaluation on the UIEB test set \cite{li2019underwater}. Best results are highlighted in {\bf bold} and second-best results are \underline{underlined}. The proposed unfolding method achieves the best performance across all metrics. Our variational model outperforms classical techinques and remains competitive with most deep learning-based approaches.}
\label{tab:UIEB_metrics}
\end{table}

\begin{table}[t]
\centering
\begin{tabular}{lccc} Method   & PSNR $\uparrow$ & SSIM $\uparrow$ &  LPIPS $\downarrow$\\ \hline Classical  \\  \hline
 CBAF \cite{ancuti2017} & 17.47 & 0.8742 & 0.191 \\ 
 BayesianRetinex \cite{zhuang2021bayesian}  &  18.04&  0.7992&  0.277
\\ 
UNTV \cite{xie2022variational} &  15.48&  0.5587&  0.335
\\ 
MMLE \cite{zhang2022underwater} &  17.86&  0.7582&  0.224
\\ 
PCDE \cite{zhang2023underwater} &  15.62&  0.6809&  0.292
\\ 
 HFM \cite{an2024hfm}  &  17.15&  0.8220&  0.218
\\ 
WWPF \cite{zhang2024underwater} &  18.09&  0.7824&  0.214
\\
Variational Model (Ours) &  20.51 &  0.8913 &  0.138
\\
\hline       Deep learning \\ \hline
Shallow-UWNet \cite{naik2021shallow} &  19.24&  0.8793&  0.126
\\ 
UColor \cite{li2021underwater} &  21.21&  0.8847&  0.128
\\ 
UIEC\^{}2-Net\cite{wang2021uiec}  &  22.29&  0.9181&  0.097
\\
AutoEnhancer \cite{tang2022autoenhancer} &  19.89&  0.8712&  0.158
\\ 
 UnfoldNet \cite{pham2023deep} &  21.98&  0.8983&  0.168
\\ 
P2CNet \cite{rao2023deep} &  18.89&  0.8179&  0.170
\\ 
 LiteEnhanceNet \cite{zhang2024liteenhancenet} &  23.08&  0.9084&  0.100
\\ 
RAUNE-Net\cite{peng2023raune}  &  22.90&  0.9099&  0.102
\\ 
  Mamba-UIE \cite{zhang2024mamba}  &  \underline{23.24} &  \underline{0.9258} &  0.104
\\ 
 Phaseformer \cite{khan2025phaseformer} &  22.86&  0.9067&  \underline{0.093}
\\ 
 UIR-PolyKernel \cite{guo2025underwater} &  22.42&  0.9225&  0.104\\\hline
  Ours &  \bf{25.58} &  \bf{0.9356} &  \bf{0.078}
\\ 
\end{tabular}
\caption{Quantitative evaluation on the SUIM-E test set \cite{qi2022sguie}. Best results are highlighted in {\bf bold} and second-best results are \underline{underlined}. The proposed unfolding method achieves the best performance across all metrics. Our variational model outperforms classical techniques, but remains below recent deep learning approaches.}
\label{tab:SUIME_metrics}
\end{table}

\begin{figure}[t]
\centering
\begin{tabular}{c@{\hskip 0.2em} c} 

\spyimageRetinex{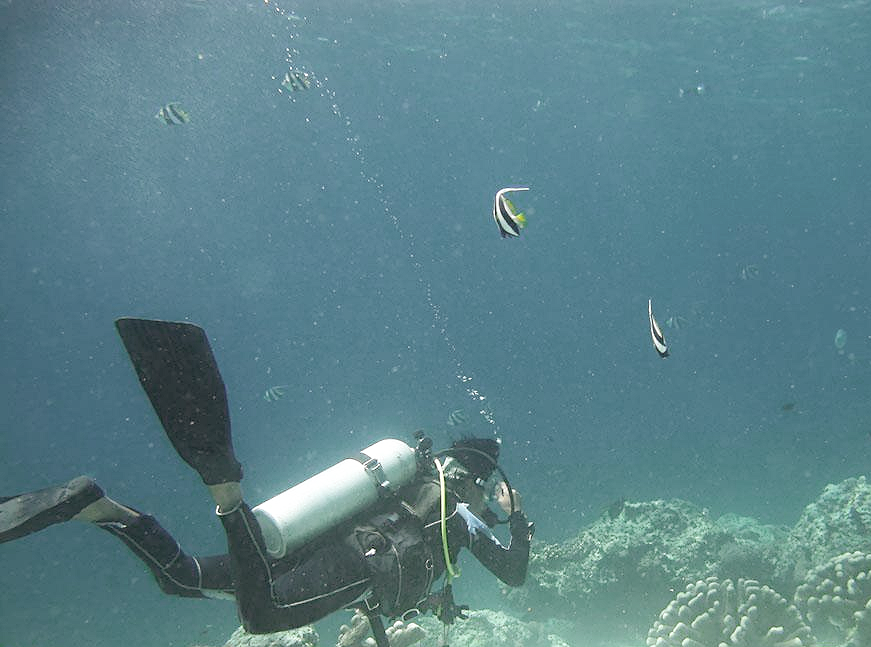} &
    \spyimageRetinex{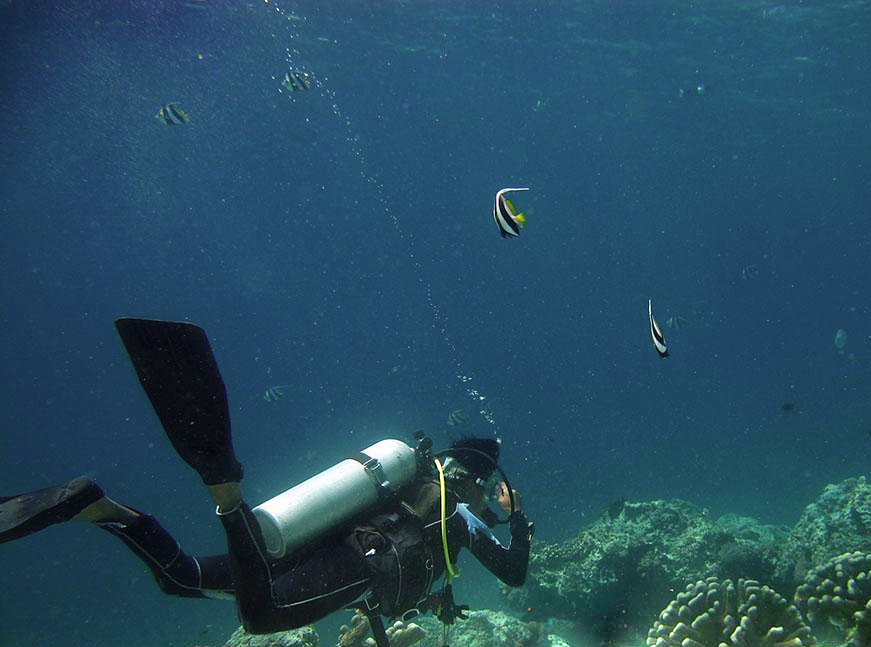} \\
      Retinex model &  Dehazing model \\
\end{tabular}
\caption{Comparison between Retinex-based and dehazing-based formulations. The latter provides superior enhancement, improving visibility, contrast, and color accuracy.}
\label{fig:abl_model_comparison}
\end{figure}

\begin{figure}[t]
\centering
\begin{tabular}{c@{\hskip 0.2em} c}     \spyimagegradfidel{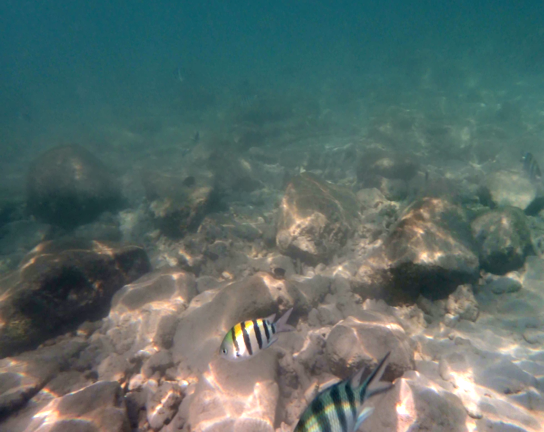} & \spyimagegradfidel{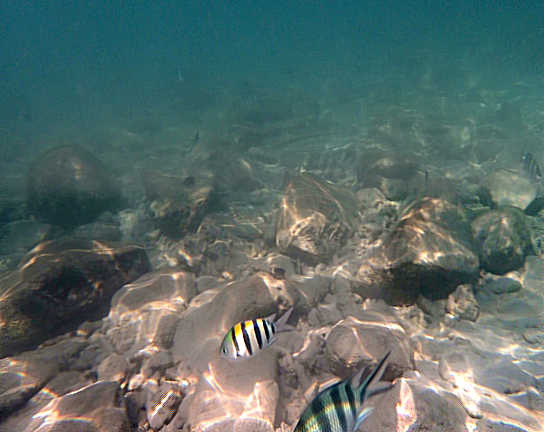} \\
      No gradient fidelity &  Ours \\
\end{tabular}
\caption{Effect of the nonlocal gradient-type constraint in \eqref{energia} for underwater image enhancement. Enforcing the constraint significantly enhances edge contrast, producing sharper details. }\label{fig:abl_grad_fidel}
\end{figure}


\section{Experimental Results}

In this section, we assess the performance of the proposed method and compare it with state-of-the-art techniques for underwater image enhancement.  
In particular, we compare with the classical methods CBAF \cite{ancuti2017}, BayesianRetinex \cite{zhuang2021bayesian}, UNTV \cite{xie2022variational}, MMLE \cite{zhang2022underwater}, PCDE \cite{zhang2023underwater}, HFM \cite{an2024hfm}, and WWPF \cite{zhang2024underwater}; the purely deep learning techniques Shallow-UWNet \cite{naik2021shallow}, UColor \cite{li2021underwater}, UIEC\^{}2-Net\cite{wang2021uiec}, AutoEnhancer \cite{tang2022autoenhancer}, P2CNet \cite{rao2023deep}, LiteEnhanceNet \cite{zhang2024liteenhancenet}, RAUNE-Net\cite{peng2023raune}, Mamba-UIE \cite{zhang2024mamba}, Phaseformer \cite{khan2025phaseformer}, and UIR-PolyKernel \cite{guo2025underwater}; and the unfolding method UnfoldNet \cite{pham2023deep}. 
The source codes were obtained from the authors' webpages and all models were trained according to the specified configurations.

We selected the UIEB dataset \cite{li2019underwater} for training our method and all compared techniques. It consists of 890 real-world underwater images and their corresponding reference ones. A random set of 875 pairs was used for training, with all images resized to 256$\times$256, and the remaining 15 pairs were reserved for testing. For cross-dataset evaluation, we randomly selected 15 images from the SUIM-E dataset \cite{qi2022sguie}. 
Both UIEB and SUIM-E references were generated by considering enhanced images produced by 12 representative enhancement methods and selecting the best result through a volunteer-based evaluation.

\begin{figure}[t]
\centering
\begin{tabular}{c@{\hskip 0.2em}  c} 
\spyimageto{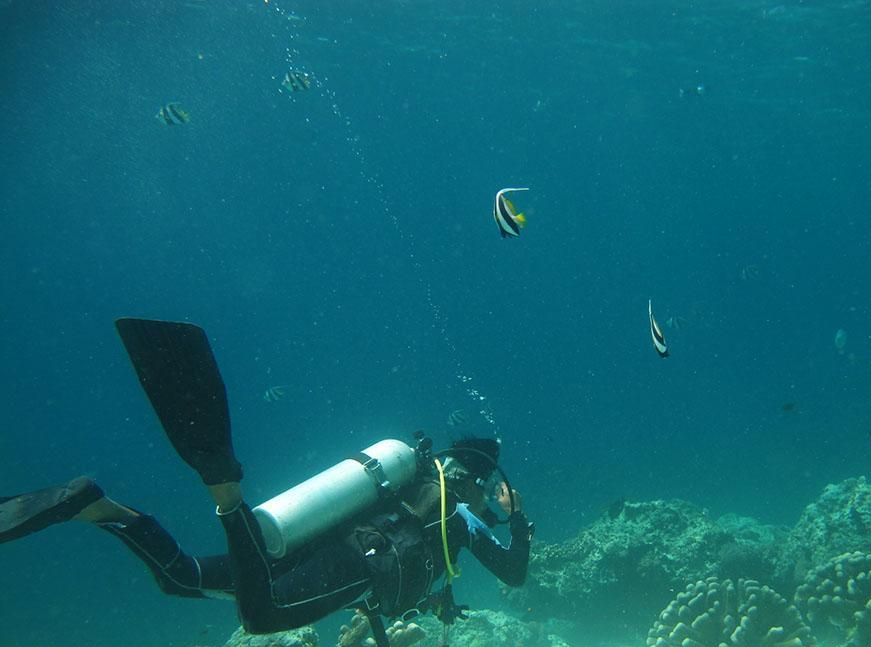}
& \spyimageto{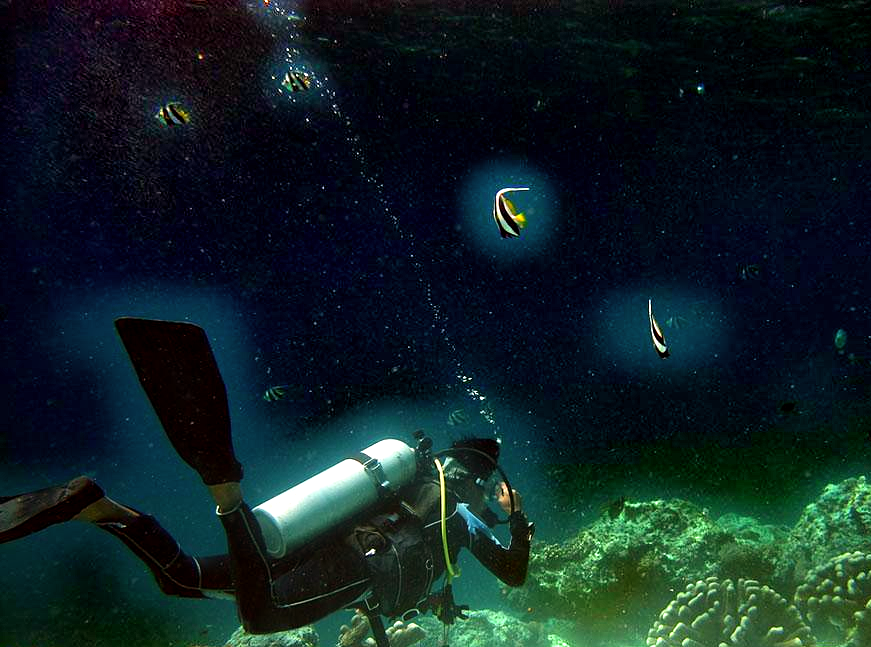}\\ 
Input & UDCP  \cite{drews2013transmission} \\ \spyimageto{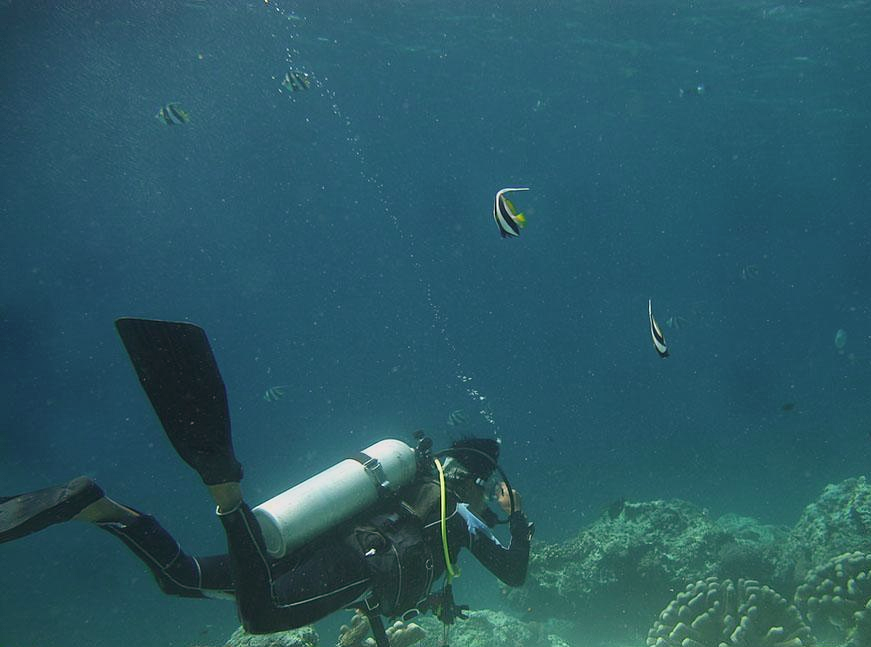} & \spyimageto{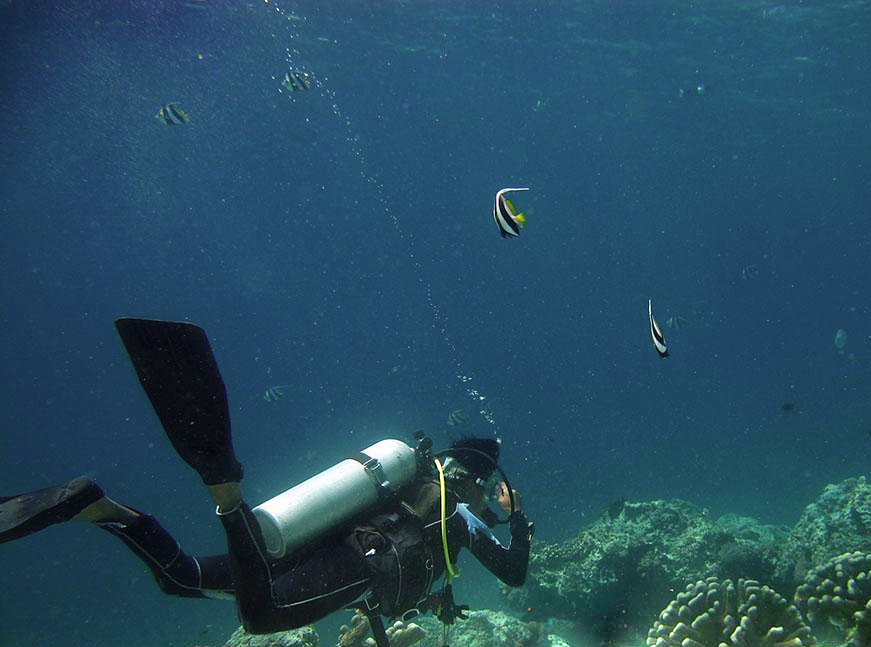}  \\
      RCP \cite{galdran2015automatic} & DCP \cite{he2010single} \\
\end{tabular}
\caption{Comparison of transmission map estimation techniques. The figure shows the final enhanced results obtained using the classical Dark Channel Prior (DCP), Red Channel Prior (RCP), and Underwater Dark Channel Prior (UDCP) as $t_0$ in \eqref{energia}.}\label{fig:abl_t0_fidel}
\end{figure}

The proposed deep unfolding network is trained  in an end-to-end manner over 500 epochs using the loss function
\begin{equation*}
    \begin{aligned}
\beta_1 \text{MSE}(J_{\text{out}}^K, J_{\text{gt}})&+\beta_2\text{LPIPS}\left(J_{\text{out}}^K, J_{\text{gt}}\right)\\&+\mathcal{L}_{\prox}(\{J_{\text{out}}^n\}_{n=1}^{K-1}, \{J_{\text{prox}}^{\frac{n\cdot S}{K}}\}_{n=1}^{K-1})\end{aligned}
\end{equation*}
where $\beta_1=0.95$, $\beta_2=0.01$, and all $\{\alpha_n\}$ of $\mathcal{L}_{\prox}$ are
fixed to 0.01. 
LPIPS \cite{LPIPS} is included to promote perceptually faithful reconstructions, as it aligns better with human visual perception than traditional metrics. We use the Adam optimizer with an initial learning rate of $10^{-4}$ and set the number of stages to $K=5$. For the Mamba modules, the patch size for extracting the non-overlapping patches is fixed to 4, and the state dimension of the SSM is set to 64. 

Table~\ref{tab:UIEB_metrics} reports the quantitative evaluation of all methods on the UIEB test set. The proposed unfolding method consistently achieves the best performance across all metrics. Our variational model, despite not relying on learning-based strategies, still outperforms classical enhancement techniques and remains competitive with most deep learning approaches. 

Figure \ref{fig:UIEBcomparison} shows cropped regions of the enhanced results produced by each method for a sample from the UIEB test set. The most common issues include color distortions, as observed in HFM and P2CNet, and insufficient haze removal, as in Shallow-UWNet and AutoEnhancer. The most satisfactory results are obtained by UColor, RAUNE-Net, Phaseformer and our two proposed approaches. However, UColor and Phaseformer introduce excessive reddish tones, particularly noticeable at the bottom of the image. Furthermore, as shown in the zoomed-in region, fish and seabed details appear clearer with our unfolding method, while Mamba-UIE and RAUNE-Net still retain a slight greenish tone in the seabed background.

\begin{table}[t]\centering
\begin{tabular}{l |c | c | c }
   Loss function  & PSNR & SSIM & LPIPS \\
    \hline
    MSE & 23.49 & 0.8885 &	0.138 
 \\
     MSE + $\mathcal{L}_{\prox}$ & 23.87 &	0.8960 & 0.112
 \\
      MSE + $\mathcal{L}_{\prox}$ + LPIPS &  \bf{25.16} &  \bf{0.9140} &  \bf{0.089} 
 \\   
\end{tabular}
\caption{Quantitative comparison of the performance of the proposed unfolding method using different loss functions. Best values are highlighted in bold.}
\label{tab:abl_loss}\end{table}

\begin{table}[t]\centering
\begin{tabular}{l| c |c | c | c }
   Architecture & Auxiliary branches & PSNR & SSIM & LPIPS \\
    \hline
    MHA & \cmark & 23.72 & 0.9038 &	0.099 
 \\\hline
 \multirow{2}{*}{Mamba} & \xmark & 22.71  & 0.8940 & 0.106 \\
       & \cmark & \bf{25.16} &  \bf{0.9140} &  \bf{0.089} 
 \\   
\end{tabular}
\caption{Comparison of MHA and Mamba architectures with and without the proposed auxiliary branches. Best values are highlighted in bold.}
\label{tab:abl_arch}\end{table}

\begin{figure}[t]
\centering
\begin{tabular}{c@{\hskip 0.2em} c} 
\spyimageMHA{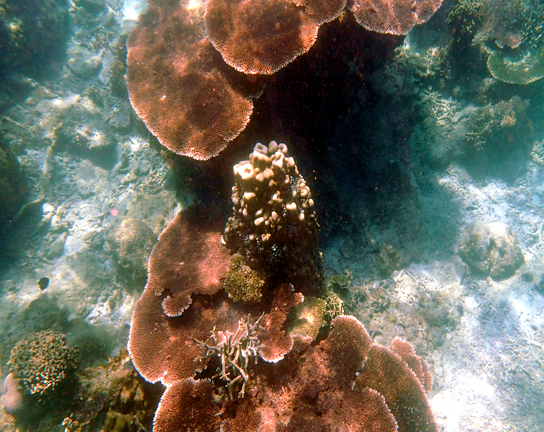}  &
    \spyimageMHA{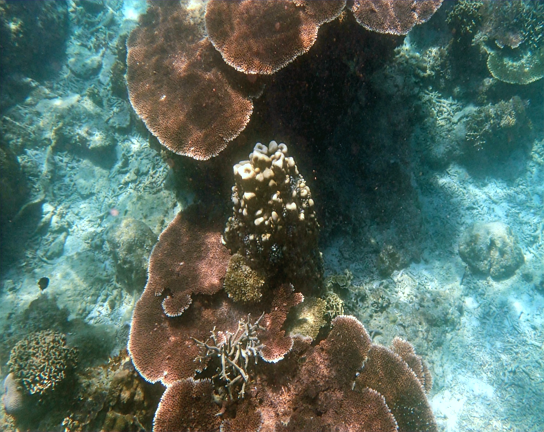}\\
    Ground Truth & MHA \\
    \spyimageMHA{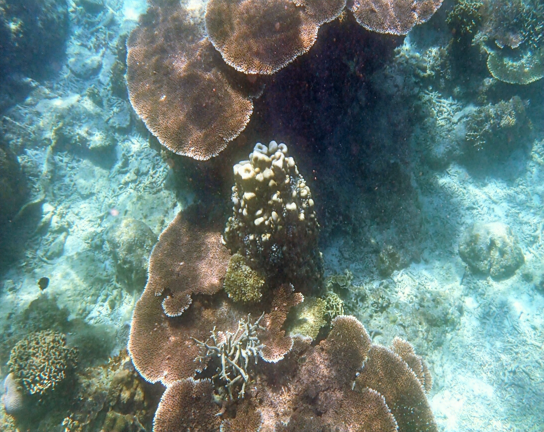} &
    \spyimageMHA{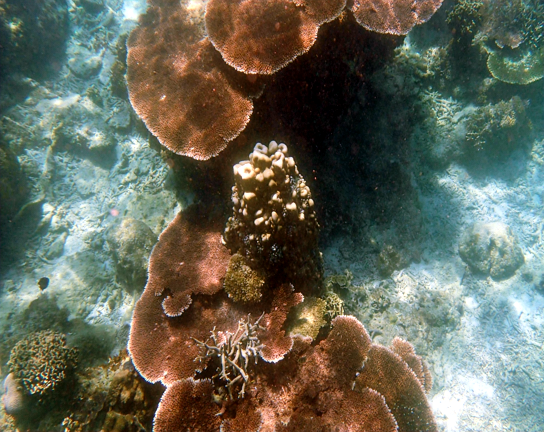} \\
     Mamba w/o auxiliary branches &  Mamba \\
\end{tabular}
\caption{Visual comparison between unfolding networks using Mamba layers, with and without auxiliary branches, and MHA. The full Mamba-based model produces results with higher contrast and more faithful color restoration.}
\label{fig:abl_arch_image}
\end{figure}

Table \ref{tab:SUIME_metrics} presents the average metrics on the SUIM-E test set. The proposed unfolding method ranks first across all evaluation metrics. Compared with classical methods, our variational model achieves better results, yet it is outperformed by recent deep learning techniques.

As illustrated in Figure \ref{fig:SUIMEcomparison}, classical approaches often produce unnatural colors in the water background. Although our variational method mitigates this issue, the resulting image exhibits a grayish tone. Several deep-learning techniques also retain greenish colors, whereas LiteEnhanceNet, Phaseformer, Mamba-UIE, and our unfolding method produce the best visual results. However, Phaseformer shows a predominance of the blue channel, while LiteEnhanceNet and Mamba-UIE generate a slightly grayish background compared with the ground-truth. In addition, these methods fail to fully recover the orange and white regions of the fish in the zoomed area.

\section{Ablation Study}
In this section, we conduct an ablation study to analyze the influence of the different terms in the variational model \eqref{energia} and to evaluate how the proposed components of the unfolding network affect both performance and computational efficiency.

We first investigate whether a dehazing-based formulation is more suitable than a Retinex-based one for underwater image enhancement. To this end, we apply the proposed energy using a Retinex decomposition based on the equivalence \eqref{retinex_dehazing}. As shown in Figure \ref{fig:abl_model_comparison}, the dehazing version consistently produces superior enhancement results, providing better visibility, contrast, and color accuracy. 



\begin{figure*}[t]
\centering
\includegraphics[width=0.85\linewidth]{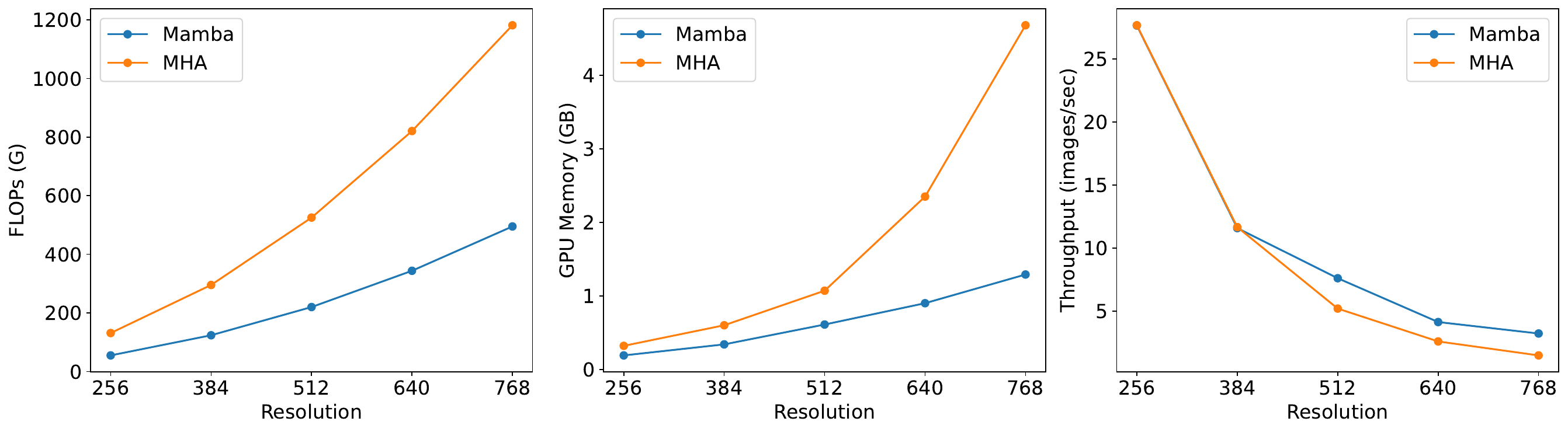}
\caption{Efficiency comparison between Mamba and MHA across different image resolutions. From left to right: FLOPs, GPU memory consumption, and inference throughput. The Mamba-based network achieves lower computational cost and memory usage, while providing higher throughput at larger image resolutions.}
\label{fig:efficiency_comparison}
\end{figure*}

We also evaluate the relevance of the nonlocal gradient-type constraint in Figure~\ref{fig:abl_grad_fidel} by comparing results obtained with and without it. Enforcing this constraint significantly enhances edge contrast, producing sharper details, 
and improves perceptual quality.

The $\rho$-term is mainly introduced in \eqref{energia} to ensure the convergence of the optimization process. Nevertheless, an accurate transmission map estimation is also essential for a correct performance. Several techniques have been proposed, with the Dark Channel Prior (DCP) \cite{he2010single} and its underwater variants, such as the Red Channel Prior (RCP) \cite{galdran2015automatic} and the Underwater Dark Channel Prior (UDCP) \cite{drews2013transmission}, being the most widely used. In our framework, the original DCP provides the most reliable transmission estimates. As shown in Figure \ref{fig:abl_t0_fidel}, UDCP tends to introduce halo artifacts and excessively darkens the water regions, whereas RCP only partially removes the haze present in the original image. These observations support the use of the classical DCP as the initial transmission estimate.

We next analyze the components of the unfolding network. As illustrated in Figure \ref{fig:abl_stages}, supervising the intermediate stages plays a crucial role. When only the MSE on the final output is used, the intermediate stages fail to produce meaningful reconstructions. In contrast, with the proposed proximal trajectory supervision, the intermediate outputs progressively evolve toward the reference image. This behavior leads to improved final results, both visually, particularly in the water regions of the example image, and quantitatively, as reported in Table \ref{tab:abl_loss}. In addition, we incorporate a small weighting of the LPIPS \cite{LPIPS} metric into the final loss function to further refine the perceptual quality of the enhanced images. The resulting gains are also reflected in the metrics reported in the same table.

Finally, we analyze performance and efficiency by comparing our Mamba-based network with a counterpart based on Multi-Head Attention (MHA) \cite{attention}. Both architectures are designed with a similar number of parameters to ensure a fair comparison. We also evaluate the impact of incorporating auxiliary branches. As shown in Table \ref{tab:abl_arch} and Figure \ref{fig:abl_arch_image}, adding auxiliary branches improves the performance of the Mamba-based model, which, in turn, consistently outperforms MHA across all metrics while producing visually superior results.

Beyond performance, efficiency is a key motivation for adopting Mamba instead of MHA for nonlocal modeling. Figure \ref{fig:efficiency_comparison} compares the computational cost at different image resolutions. The results show that both FLOPs and GPU memory consumption are consistently higher for MHA and increase rapidly with image resolution, particularly for memory usage. In terms of throughput, both approaches have similar execution times at low resolutions. However, as the image size increases, the Mamba-based network scales more efficiently and processes a larger number of images per second.

\section{Conclusions}


In this work, we have presented a deep unfolding framework for underwater image enhancement built upon a dehazing-based variational model. The formulation incorporates a multiplicative residual component to absorb remaining artifacts and a nonlocal gradient-type constraint to enhance edge sharpness, leading to improved restoration quality. In addition, we have provided a theoretical foundation for the proposed approach by proving the existence of a minimizer of the variational model. Building upon this formulation, we have derived a learnable deep unfolding architecture in which the proximal operators are replaced by residual neural networks and Mamba-based modules to efficiently model nonlocal interactions. Furthermore, a proximal trajectory loss is introduced to enforce consistency between the unfolding stages and the iterations of an ideal restoration regularizer, thereby preserving the interpretability of the underlying optimization process.

Experiments on different datasets have demonstrated that the proposed unfolding network improves visual quality and quantitative performance compared with recent state-of-the-art techniques. 
Ablation studies have further confirmed the importance of both the variational formulation and the proposed network components in achieving high-quality restoration.

Despite these promising results, several directions remain open for future research. The proposed framework still relies on an initial transmission estimate, and its performance could benefit from more accurate physically motivated priors or jointly learned initialization strategies. Furthermore, extending the model to video enhancement and temporal consistency, as well as exploring alternative physically grounded degradation models, constitutes an interesting direction for future work.

\bibliographystyle{IEEEtran}
\bibliography{refs}

\end{document}